\documentclass[journal,10pt]{IEEEtran}

\usepackage{tabularx}
\usepackage{array}

\usepackage{amsmath,amsfonts,bm,bbm,enumitem}
\allowdisplaybreaks
\usepackage{color}
\usepackage{amsthm}
\usepackage{graphicx}
\usepackage{url}
\usepackage{algorithm,algorithmic,multirow,hhline}
\usepackage{dblfloatfix}
\usepackage{amssymb}
\usepackage{caption}
\usepackage{comment}
\usepackage{subfigure}
\usepackage{pstricks}
\usepackage{setspace}
\usepackage{color}
\usepackage{booktabs,multirow}
\usepackage{enumitem}
\usepackage{multicol}
\input{mysymbol.sty}

\usepackage{cite}

\def\ie{\emph{i.e.,\ }}

\def\sm{\small}
\def\nm{\normalsize}

\providecommand{\ip}[1]{\boldsymbol{\langle}#1\boldsymbol{\rangle}}

\def\lamb{\boldsymbol{\lambda}}

\makeatletter

\newcommand{\Rmnum}[1]{\expandafter\@slowromancap\romannumeral #1@}
\makeatother

\newtheorem{theorem}{Theorem}
\newtheorem{lemma}{Lemma}

\newtheorem{corollary}{Corollary}

\providecommand{\propositionname}{Proposition}

\usepackage[T1]{fontenc}
\usepackage{etoolbox}

\title{Harnessing Wireless Channels for Scalable and Privacy-Preserving Federated Learning}

\author{Anis Elgabli, $^\dagger$Jihong Park, Chaouki Ben Issaid, and Mehdi Bennis \thanks{A. Elgabli, C. Ben Issaid, and M. Bennis are with the Centre of Wireless Communications, University of Oulu, 90014 Oulu, Finland, email: \{anis.elgabli, chaouki.benissaid, mehdi.bennis\}@oulu.fi.}\\ 
\thanks{$^\dagger$J. Park is with the School of Information Technology, Deakin University, Geelong, VIC 3220, Australia, email: jihong.park@deakin.edu.au.}}

\begin{document}

\maketitle

\begin{abstract}
Wireless connectivity is instrumental in enabling scalable federated learning (FL), yet wireless channels bring challenges for model training, in which channel randomness perturbs each worker's model update while multiple workers' updates incur significant interference under limited bandwidth. To address these challenges, in this work we formulate a novel constrained optimization problem, and propose an FL framework harnessing wireless channel perturbations and interference for improving privacy, bandwidth-efficiency, and scalability. The resultant algorithm is coined \emph{analog federated ADMM (A-FADMM)} based on analog transmissions and the alternating direction method of multipliers (ADMM). In A-FADMM, all workers upload their model updates to the parameter server (PS) using a single channel via analog transmissions, during which all models are perturbed and aggregated over-the-air. This not only saves communication bandwidth, but also hides each worker's exact model update trajectory from any eavesdropper including the honest-but-curious PS, thereby preserving data privacy against model inversion attacks. We formally prove the convergence and privacy guarantees of A-FADMM for convex functions under time-varying channels, and numerically show the effectiveness of A-FADMM under noisy channels and stochastic non-convex functions, in terms of convergence speed and scalability, as well as communication bandwidth and energy efficiency.
\end{abstract}


\section{Introduction}
Wireless connectivity has a great potential to scale up federated learning (FL) \cite{Brendan17,pap:jakub16,Google:FL19} by cutting the wires between workers and their parameter server (PS) \cite{park2018wireless, FL_Nishio,Wang:2019aa,YangQuekPoor:2019aa,Chen:20019aa,Amiri:SPAWC19,Zhu:19,Sery:19,Zhu:20,park2020:cml}. However, wirelessly connected workers may \emph{interfere} with each other during their over-the-air transmissions, while competing over limited bandwidth. Most existing works avoid such interference by allocating dedicated channels to different workers \cite{park2018wireless,FL_Nishio,Wang:2019aa,YangQuekPoor:2019aa,Chen:20019aa}, which is not scalable and requires significant amounts of bandwidth to support a large number of workers. Alternatively, taking a cue from FL operations, several recent works have proposed a method harnessing interference without separate channel allocation \cite{Amiri:SPAWC19,Zhu:19,Sery:19,Zhu:20,park2020:cml} as we review next.

As illustrated in Fig.~\ref{Fig:overview}(a), FL aims to minimize  {\sm$\sum_{n=1}^N f_n(\boldsymbol{\Theta})$\nm}  assuming {\sm$N$\nm} workers, by periodically uploading the local model {\sm$\boldsymbol{\theta}_n$\nm} (or local gradient {\sm$\nabla f_n(\boldsymbol{\Theta})$\nm})  of each worker and downloading a global model {\sm$\boldsymbol{\Theta}$\nm} from the PS. Under digital transmissions, i.e, \emph{digital FL}, $(i)$ the PS first receives each {\sm$\boldsymbol{\theta}_n$\nm} through a \emph{separate channel per each worker}, and $(ii)$ combines them into a global model {\sm$\boldsymbol{\Theta}=\frac{1}{N}\sum_{n=1}^N \boldsymbol{\theta}_n$\nm}. The first step is however vulnerable to model inversion and reconstruction attacks \cite{Matt:CCS15,Hitaj:CCS17} by an honest-but-curious PS. Since the entire model update trajectory is observable, the PS can infer the training samples, violating data privacy. Furthermore, it is not communication-efficient because workers have to be assigned orthogonal channels in order for the PS to decode their models. However, the PS only needs {\sm$\sum_{n=1}^N \boldsymbol{\theta}_n$\nm} rather than individual local models, motivating the need for analog over-the-air aggregation schemes as described next.

Unlike digital signal transmission of bit streams, each analog signal directly represents an element $\boldsymbol{\theta}_{n,i}$ of $\boldsymbol{\theta}_n$ by its amplitude, allowing signal superposition. Exploiting this property, each worker in \emph{analog FL} transmits an analog signal $\boldsymbol{\theta}_{n,i}$ over a \emph{shared channel among all workers}, through which all $\boldsymbol{\theta}_{n,i}$'s are superpositioned over the air while hiding each private local model in the crowd. Consequently, the PS receives $\sum_{n=1}^N h_{n,i} \boldsymbol{\theta}_{n,i}$ that is \emph{perturbed} by complex fading channel $h_{n,i}$. Due to the perturbed models, the convergence and accuracy of analog FL depend significantly on the channel characteristics. To obviate this problem, it is common to cancel out the perturbation via a \emph{channel inversion} method dividing $\boldsymbol{\theta}_{n,i}$ by $h_{n,i}$ before transmissions, as illustrated in Fig.~\ref{Fig:overview}(b). With channel inversion, transmissions are only allowed when $|h_{n,i}|^2\geq \varepsilon$, in order to avoid excessive transmit power due to the inversion \cite{Amiri:SPAWC19,Zhu:19,Sery:19}. The choice of $\varepsilon$ is heuristic, hindering the convergence analysis of analog FL. Moreover, this approach does not guarantee privacy. For example, when only one worker has a good gain in one channel, it reveals its local model updates to the PS, compromising privacy. Last but not least, the rule of transmitting only when $|h_{n,i}|^2\geq \varepsilon$ totally ignores the power of the transmitted symbol itself. Note that we are sending analog signals, and hence these limitations mandate a non-channel inversion method with formal convergence and privacy guarantees.


\begin{figure*}[t]
    \centering
    \includegraphics[width=.99\textwidth]{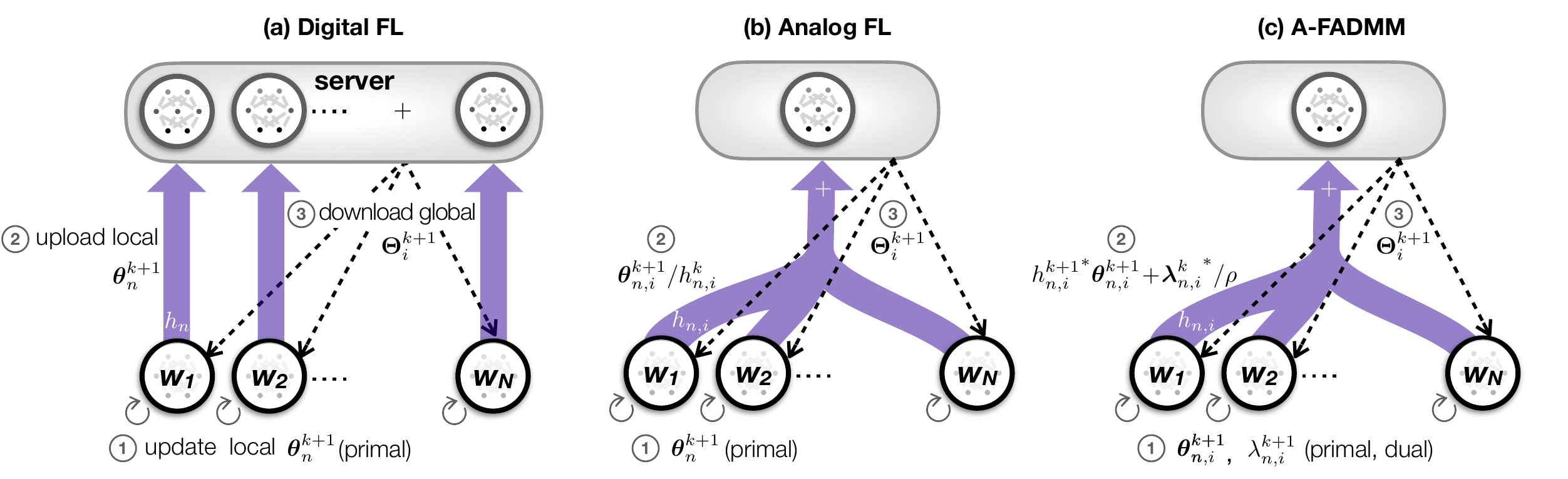}
    \caption{Schematic illustrations of: (a) digital federated learning (FL), (b) analog FL with channel inversion, and (c) \emph{analog-federated ADMM (A-FADMM)} without channel inversion.}
    \label{Fig:overview}
    \vskip -10pt
    \end{figure*} 
    
In this paper, we propose A-FADMM, a novel federated learning framework rooted in the alternating direction method of multipliers (ADMM) and analog over-the-air aggregation without channel inversion. Compared to the existing analog FL algorithms \cite{Amiri:SPAWC19,Zhu:19,Sery:19} based on first-order methods such as GD and SGD, A-FADMM is a \emph{second-order} method providing faster convergence \cite{boyd2011distributed,deng2017parallel}. Furthermore, A-FADMM does not apply channel inversion, so the PS receives the aggregate of perturbed model updates, thereby preserving privacy even when a single worker is transmitting over a given channel. This is done by integrating channel perturbations into the problem formulation, which may hamper the convergence particularly under time-varying channels. A-FADMM thus carefully updates the model parameters so that the time-varying channel does not hinder its convergence. Our major contributions are summarized as follows.
\begin{itemize}[leftmargin=*]
    \item On the theory front, this is the first work on analog transmission based distributed ADMM (primal-dual method) ensuring convergence while preserving privacy, under convex functions and time-varying channel dynamics. Existing works on analog FL focus on first-order primal methods, without proving convergence nor privacy guarantees \cite{Amiri:SPAWC19,Zhu:19,Sery:19}.   
    
    \item On the algorithmic front, our proposed A-FADMM is the first analog FL algorithm overcoming channel perturbations without channel inversion while ensuring convergence over time-varying channels. 
    
    

    \item We numerically show that A-FADMM converges faster with comparable accuracy, compared to its digital transmission counterpart D-FADMM. Our simulations clarify under which conditions A-FADMM is preferable to D-FADMM, in terms of energy-efficiency, low-latency, and scalability. 
    
    \item To further support the feasibility of A-FADMM, we elaborate on how to cope with constrained transmit power. Moreover, to corroborate the feasibility under stochastic and non-convex functions, we provide simulation results for the stochastic version of A-FADMM (SA-FADMM) based on a deep neural network (DNN) in an image classification task.
\end{itemize}

The rest of the paper is structured as follows. In Section~II, the proposed Analog Federated ADMM (A-FADMM) is described. The convergence of the proposed algorithm is studied in Section~III, and the privacy analysis is provided in Section~IV. The effectiveness of A-FADMM is numerically corroborated in Section~V, in terms of accuracy and communication efficiency for linear regression and image classification using DNNs, followed by our conclusion in Section~VI.


\section{Analog Federated ADMM}\label{SecPeoposedAnalog}
A-FADMM aims to aggregate multiple workers' updates at the PS without competition on the available bandwidth via analog transmissions. In this section, we describe A-FADMM operations based on a novel problem formulation, and explain how A-FADMM copes with the nuisances incurred by analog transmissions, in terms of time-varying channel fading, noise, and transmit power limitation.



\subsection{Problem Formulation} 
The original problem of FL is to minimize $\frac{1}{N}\sum_{n=1}^N f_n(\boldsymbol{\Theta})$ with $N$ workers, by locally minimizing $f_n(\boldsymbol{\theta}_n)$ at each worker and globally averaging their model
parameters $\boldsymbol{\theta}_n$ at the PS. This boils down to the average consensus problem (\textbf{P1}) below.
\begin{align} 
       \nonumber (\textbf{P1})~~ \min_{\boldsymbol{\Theta},\{\boldsymbol{\theta}_n\}_{n=1}^N}\ \ \ & \sum_{n=1}^N f_n(\boldsymbol{\theta}_n)    \\
    \text{s.t.}\ \ &
        \boldsymbol{\theta}_{n} =\boldsymbol{\Theta}, \ \ \forall n \label{P1_const}
\end{align} 

Primal-dual methods can solve (\textbf{P1}), among which ADMM is one popular approach \cite{glowinski1975approximation,boyd2011distributed,deng2017parallel}. To implement this using digital transmissions, workers transmit their local model updates to the PS through \emph{orthogonal channels}, wherein each local model is the primal variable $\boldsymbol{\theta}_{n}$ while its dual variable $\bm{\lambda}_n$ is locally updated. Next, we explain the steps of the standard ADMM technique \cite{glowinski1975approximation,boyd2011distributed,deng2017parallel} in solving ({\bf P1}).

 The augmented Lagrangian of ({\bf P1}) is written as
\begin{align}
\nonumber &\boldsymbol{\mathcal{L}_\rho}(\boldsymbol{\Theta},\{\boldsymbol{\theta}_n\}_{n=1}^N,\{\boldsymbol{\lambda}_n\}_{n=1}^N )\\
&=\sum_{n=1}^N f_n(\boldsymbol{\theta}_n) + \sum_{n=1}^N \ip{\bm{\lambda}_{n}, \boldsymbol{\theta}_n - \boldsymbol{\Theta}}+\frac{\rho}{2} \sum_{n=1}^N \parallel \boldsymbol{\theta}_n - \boldsymbol{\Theta}\parallel_2^2,
\label{augmentedLag4}
\end{align}
where $\rho>0$~is a constant penalty for the disagreement between $\bbtheta_n$~and $\boldsymbol{\Theta}$. At iteration $k+1$, each worker updates its primal variable by solving the following problem
\small
\begin{align}
{\boldsymbol{\theta}}_{n}^{k+1} =\arg\min_{\bbtheta_n}\Big\{f_n(\boldsymbol{\theta}_n) +&\ip{\boldsymbol{\lambda}_n^k, \boldsymbol{\theta}_n - \boldsymbol{\Theta}^{k}}+\frac{\rho}{2}\parallel \boldsymbol{\theta}_n - \boldsymbol{\Theta}^k\parallel_2^2\Big\}.
\label{headUpdate}
\end{align}
\normalsize
Based on all workers' primal variable updates $\{\boldsymbol{\theta}_n^{k+1}\}_{n=1}^N$ and previous dual variables $\{\boldsymbol{\lambda}_n^{k}\}_{n=1}^N$, PS updates the global model $\boldsymbol{\Theta}^{k+1}$ as follows
\begin{align}
\boldsymbol{\Theta}^{k+1}=\frac{1}{N}\sum_{n=1}^N ({\boldsymbol{\theta}}_{n}^{k+1}+\frac{1}{\rho}\bm{\lambda}_n^k).
\end{align}
Finally, given the updated global model $\boldsymbol{\Theta}^{k+1}$, each worker updates the dual variable $\bm{\lambda}_n^{k+1}$ as follows
\begin{align}
\bm{\lambda}_n^{k+1}=\bm{\lambda}_n^{k}+\rho(\boldsymbol{\theta}_n^{k+1} - \boldsymbol{\Theta}^{k+1}).
\label{dualUpdate1}
\end{align}
To implement ADMM operations, each worker uploads $\boldsymbol{\theta}_n^{k+1}$ to the PS, and then downloads $\bm{\Theta}^{k+1}$ from PS, followed by locally updating the dual variable $\bm{\lambda}^{k+1}$. Under digital transmission, the entire bandwidth is orthogonally allocated to each worker, while each update uploading or downloading corresponds to exchanging a fixed number of bits, e.g., 32 bits per model's element. To cope with channel fading and noise, adaptive modulation and error-correction coding are used in digital transmission. 


By contrast, using analog transmissions, $N$ workers transmit the $i$-th element of their updates using the same $i$-th subcarrier (channel). The benefit of analog transmissions is to aggregate all workers' updates over-the-air in one channel use, but at the cost of perturbations by channel fading $h_{n,i}$ which are assumed to follow an independent and identically distributed (IID) complex Gaussian distribution. These fading perturbations are often cancelled by multiplying $1/h_{n,i}$ before transmission, i.e., channel inversion \cite{Amiri:SPAWC19,Zhu:19,Sery:19}. Alternatively, we avoid channel inversion by reformulating (\textbf{P1}) into (\textbf{P2}) below, where the subscript $i$ denotes the $i$-th element. 
\begin{align}
        (\textbf{P2})~~ \min_{\boldsymbol{\Theta}, \{\boldsymbol{\theta}_n\}_{n=1}^N} \ & \sum_{n=1}^N f_n(\boldsymbol{\theta}_n)      \label{com_agadmm}\\
        \text{s.t.}\ \ &
        h_{n,i}\boldsymbol{\theta}_{n,i} = h_{n,i}\boldsymbol{\Theta}_{i}, \  \forall n,i
        \label{com_agadmm_c1}
\end{align} 

Note that, in (\textbf{P2}), we assume that the channel is constant and the system is noise free. We will relax these assumptions later in the section.  In (\textbf{P2}), constraint \eqref{P1_const} is recast as its equivalent constraint \eqref{com_agadmm_c1} that allows A-FADMM to be updated directly using perturbed updates, as detailed next.

\subsection{Primal, Dual, and Global Model Updates}
The Lagrangian of (\textbf{P2}) is written as follows
\begin{align}
\nonumber \boldsymbol{\mathcal{L}_\rho}(\boldsymbol{\theta}_n,\boldsymbol{\lambda})&=\sum_{n=1}^N f_n(\boldsymbol{\theta}_n) + \sum_{i=1}^d \sum_{n=1}^N \lamb _{n,i}^* h_{n,i}(\boldsymbol{\theta}_{n,i} - \boldsymbol{\Theta}_{i})\\
&+\frac{\rho}{2} \sum_{i=1}^d \sum_{n=1}^N |h_{n,i}|^2(\boldsymbol{\theta}_{n,i} - \boldsymbol{\Theta}_i)^2,
\label{augmentedLagAG}
\end{align} 
where $ \lamb _{n,i}^*$ is the conjugate of the complex dual variable $\lamb _{n,i}$, $d$ is the cardinality of ${\boldsymbol{\theta}}_{n}$ (i.e., model size), and $\rho>0$ is a constant penalty for the local and global model disagreement. At iteration $k+1$, each worker updates its primal variable ${\boldsymbol{\theta}}_{n}^{k+1}$ so as to minimize $\boldsymbol{\mathcal{L}_\rho}(\boldsymbol{\theta}_n,\boldsymbol{\Theta}^k,\boldsymbol{\lambda}^k)$. Hence, ${\boldsymbol{\theta}}_{n,i}^{k+1}$, $~ \forall i = 1,\cdots, d$, should satisfy the following equation
\begin{align}
\boldsymbol{0} \in \partial_i f_n(\boldsymbol{\theta}_n^{k+1}) + {\lamb _{n,i}^k}^* h_{n,i} +\rho |h_{n,i}|^2 (\boldsymbol{\theta}_{n,i}^{k+1} -\boldsymbol{\Theta}_i^k),
\label{workerUpdate0}
\end{align}
where $\partial_i f_n(\cdot)$ denotes the $i$-th element in the sub-gradient vector of $f_n(\cdot)$. 

Next, PS updates the $i$-th element ${\boldsymbol{\Theta}}_{i}^{k+1}$ of the global model that minimizes $\boldsymbol{\mathcal{L}_\rho}(\boldsymbol{\theta}_n^{k+1},\boldsymbol{\Theta},\boldsymbol{\lambda}^k)$. By taking the derivative of $\boldsymbol{\mathcal{L}_\rho}(\boldsymbol{\theta}_n^{k+1},\boldsymbol{\Theta},\boldsymbol{\lambda}^k)$ with respect to ${\boldsymbol{\Theta}}_{i}$ and equating to zero, ${\boldsymbol{\Theta}}_{i}^{k+1}$ is given~by
\begin{align}
\boldsymbol{\Theta}_i^{k+1}=\frac{1}{ \sum_{n=1}^N |h_{n,i}|^2} \sum_{n=1}^N \left(|h_{n,i}|^2{\boldsymbol{\theta}}_{n,i}^{k+1} + h_{n,i} {\bblambda_{n,i}^k}^*/\rho  \right).
\label{centralUpdate}
\end{align}

%

Finally, the dual variables are updated at each worker as follows
\begin{align}
{\lamb _{n,i}^{k+1}}&= {\lamb _{n,i}^{k}}+\rho h_{n,i} (\boldsymbol{\theta}_{n,i}^{k+1} - \boldsymbol{\Theta}_{i}^{k+1}).
\label{dualUpdate0}
\end{align}
Next, we will discuss how to implement the aforementioned update rules under time-varying channel fading, noise, and transmit power limitation.

\subsection{Time-varying Channel}
The primal-dual update rules in \eqref{workerUpdate0} and \eqref{dualUpdate0} do not ensure the non-increase of the optimality gap when $h_{n,i}^{k+1}\neq h_{n,i}^{k}$. In this case, instead of updating $\boldsymbol{\theta}_{n,i}^{k+1}$ using \eqref{workerUpdate0}, we choose $\boldsymbol{\theta}_{n,i}^{k+1}=\boldsymbol{\theta}_{n,i}^{k}$, and find ${\lamb_{n,i}^{k}}^*$ that satisfies \eqref{workerUpdate0}; in other words, the primal update problem \eqref{workerUpdate0} is flipped to the dual update problem. In doing so, A-FADMM copes with the channel changes  reflected in the dual variables, and ensures that the primal-dual variables are still optimal for the given channel~$h_{n,i}^{k+1}$.







\begin{algorithm}[t]
    {\tiny 
\small
\begin{algorithmic}[1] 
 \STATE {\bf Input}: $N, f_n(\boldsymbol{\theta}_n)\forall n, \rho, K$, {\bf Output}: $\boldsymbol{\theta}_n, \forall n$
    \STATE {\bf Initialization}:  $\boldsymbol{\theta}_n^{(0)}, \boldsymbol{\Theta}_n^{(0)}, \boldsymbol{\lambda}_n^{(0)}, \forall n$

\STATE {\bf A-FADMM:}
\WHILE {$k \leq K$}
\STATE { \bf  All workers ($n \in \{1,\cdots, N\}$): in Parallel}
\FOR{$i=1,\cdots,d$}
\IF{$h_{n,i}^{k+1}= h_{n,i}^{k}$}
\STATE \quad Find $\boldsymbol{\theta}_{n,i}^{k+1}$ that satisfies~\eqref{workerUpdateNoisy}
\ELSE
\STATE $\boldsymbol{\theta}_{n,i}^{k+1}=\boldsymbol{\theta}_{n,i}^{k}$
\STATE Find $(\boldsymbol{\lambda}_{n,i}^{k})^*$ that satisfies \eqref{workerUpdateNoisy}
\ENDIF
\ENDFOR
\STATE \quad Send~${(h_{n,i}^{k+1})}^*{\boldsymbol{\theta}}_{n,i}^{k+1}+ {(\bblambda_{n,i}^k)}^*/\rho$, $\forall i=1,\cdots,d$ to the parameter server

\STATE { \bf  Parameter Server: }
\STATE \quad Find $\boldsymbol{\Theta}_i^{k+1}$ that satisfies~\eqref{centralUpdateNoisy}
\STATE \quad Broadcast~$\boldsymbol{\Theta}_i^{k+1}$, $\forall i=1,\cdots,d$ to all workers
\STATE { \bf  All workers ($n \in \{1,\cdots, N\}$): in Parallel}
\STATE \quad Update $\lamb _{n,i}^{k+1}$ locally via~\eqref{dualUpdate1Noisy}
\STATE $k \leftarrow k+1$
\ENDWHILE 

   \end{algorithmic}
\caption{Analog Federated ADMM (A-FADMM)} \label{algo1}
}						
\end{algorithm}

\subsection{Uploading and Downloading Information}
We assume that every worker knows its individual channel {\sm$h_{n,i}^{k+1}$\nm}, while the PS knows the aggregate channel {\sm$\sum_{n=1}^N |h_{n,i}^{k+1}|^2$\nm\!} using pilot signals \cite{Ozdemir:ST07}. Then, to update the $i$-th element of the global model {\sm$\bm{\Theta}_i$\nm}, each worker uploads {\sm${h_{n,i}^{k+1}}^*\bm{\theta}_{n,i}^{k+1} + {\bm{\lambda}_{n,i}^k}^*/\rho$\nm}, where $h_{n,i}^*$ is the conjugate of the complex channel $h_{n,i}$. Hence, after channel perturbation, the PS receives {\sm$\sum_{n=1}^N\!( |h_{n,i}^{k+1}|^2{\boldsymbol{\theta}}_{n,i}^{k+1} +  h_{n,i}^{k+1}{\bblambda_{n,i}^k}^*\!/\rho)$\nm} in \eqref{centralUpdate}. By downloading $\bm{\Theta}_i$, each worker locally updates the primal and dual variables using \eqref{workerUpdate0} and \eqref{dualUpdate0}, respectively.

\begin{figure*}
	\centering
	\includegraphics[width=.8\textwidth]{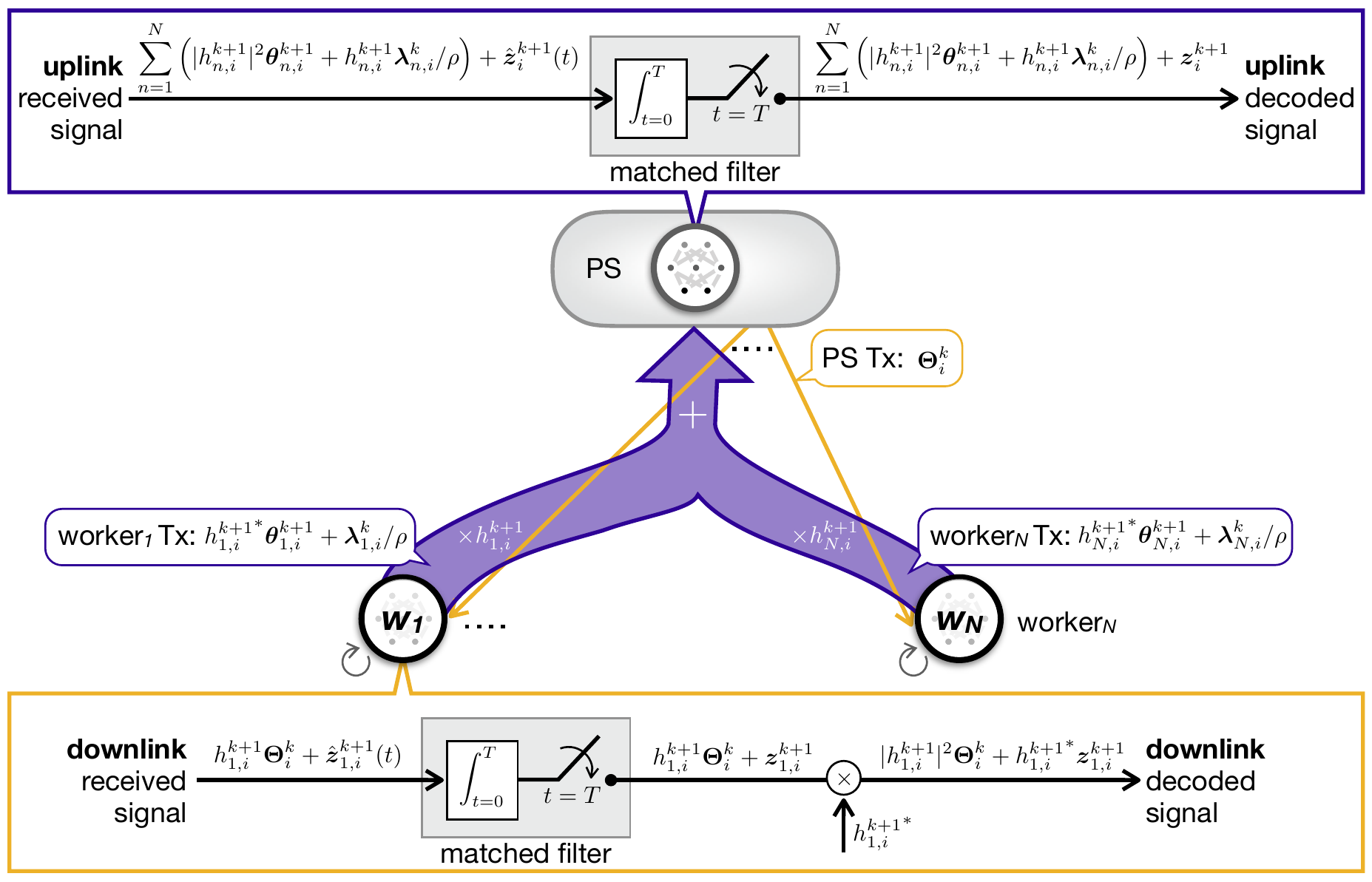}
	\caption{An illustration of uplink and downlink communication in A-FADMM under channel noise.}
	\label{commSysFig}
\end{figure*}

\subsection{Noisy Channel}
In practical systems, the received signal is not only perturbed by channel fading but also distorted by additive white Gaussian noise (AWGN). Under digital transmissions, the noise can be alleviated using digital modulation and error correction coding schemes \cite{Giannakis:04,Heath:Mag02}. By contrast, A-FADMM conveys uncoded information using analog transmissions. Therefore, the received information is perturbed by multiplicative fading and distorted by additive noise. A-FADMM directly utilizes the fading perturbed updates, yet still corrects channel noise using matched filtering (i.e., correlator receiver) as follows.

In the uplink of iteration $k+1$, as illustrated in Fig.~\ref{commSysFig}, each worker uploads its update {\sm${h_{n,i}^{k+1}}^*\bm{\theta}_{n,i}^{k+1} + {\bm{\lambda}_{n,i}^k}^*/\rho$\nm} to the PS over the $i$-th subcarrier for $T$ seconds. Propagating through the wireless channel, each update is perturbed by fading (i.e., multiplying by $h_{n,i}^{k+1}$), aggregated across all workers, and distorted by channel noise (i.e., adding $\hat{\boldsymbol{z}}_i^{k+1}(t)\sim\mathcal{CN}(0,N_o)$). Consequently, the PS receives {\sm$\sum_{n=1}^N\!( |h_{n,i}^{k+1}|^2{\boldsymbol{\theta}}_{n,i}^{k+1} +  h_{n,i}^{k+1}{\bblambda_{n,i}^k}^*\!/\rho) + \hat{\bm{z}}_{n,i}^{k+1}(t)$\nm} at every instant $t\in[0,T]$, where the AWGN {\sm$\hat{\bm{z}}_{n,i}^{k+1}(t)\sim\mathcal{CN}(0,N_o)$\nm}. The matched filter (i.e., correlator receiver) at PS integrates the received signals during~$T$, and takes a sample at $t=T$, resulting in
\begin{align}
	\nonumber &\frac{1}{T} \int_{t=0}^T\!\Big\{\!\big(\sum_{n=1}^N |{h_{n,i}^{k+1}}|^2 {\boldsymbol{\theta}}_{n,i}^{k+1}\!+\! {\bblambda_{n,i}^k}^* h_{n,i}^{k+1}/\rho\big) \!+\! \hat{\boldsymbol{z}}_i^{k+1}(t)\!\Big\} \text{dt} \\
	&= \sum_{n=1}^N |h_{n,i}^{k+1}|^2{\boldsymbol{\theta}}_{n,i}^{k+1}\!+\! {\bblambda_{n,i}^k}^*h_{n,i}^{k+1}/\rho \!+\! {\bm{z}}_{n,i}^{k+1}\!,
\end{align}
where the resultant noise ${\bm{z}}_{n,i}^{k+1} \sim\mathcal{CN}(0,N_0/T)$ whose variance is reduced from $N_0$ to $N_0/T$. Accordingly, the global model update is given by
\begin{align}
\boldsymbol{\Theta}_i^{k+1}    =\frac{\sum_{n=1}^N \left(|h_{n,i}^{k+1}|^2{\boldsymbol{\theta}}_{n,i}^{k+1} + h_{n,i}^{k+1} {\bblambda_{n,i}^k}^*/\rho + {\bf Re}\{\bm{z}_{n,i}^{k+1}\} \right)}{ \sum_{n=1}^N |h_{n,i}^{k+1}|^2},
    \label{centralUpdateNoisy}
\end{align}
where {\sm${\bf Re}\{z_{n,i}^{k+1}\}$\nm} is because {\sm$\boldsymbol{\Theta}_i^{k+1}$\nm} is real. Note that {\sm$\lamb _{n,i}^k$\nm} is complex, but the term {\sm$h_{n,i}^{k+1} {\lamb _{n,i}^k}^*$\nm} is still real. 


Likewise, in the downlink, the matched filter at each worker provides {\sm$h_{n,i}^{k+1}\bm{\Theta}_i^{k+1}+\bm{z}_{n,i}^{k+1}$\nm}. To make this output fit with the primal and dual updates, the output is multiplied by {\sm${h_{n,i}^{k+1}}^*$, and $|h_{n,i}^{k+1}|^2\bm{\Theta}_i^{k+1}+{h_{n,i}^{k+1}}^*\bm{z}_{n,i}^{k+1}$\nm} is used for the following primal update rule
\begin{align}
\nonumber &\boldsymbol{0} \in \partial_i f_n(\boldsymbol{\theta}_n^{k+1}) + (\lamb _{n,i}^k)^*h_{n,i}^{k+1}\\
& +\rho |h_{n,i}^{k+1}|^2( \boldsymbol{\theta}_{n,i}^{k+1} - \boldsymbol{\Theta}_i^k)-\rho {\bf Re}\{{h_{n,i}^{k+1}}^* \bm{z}_{n,i}^{k+1}\},
    \label{workerUpdateNoisy}
\end{align}
and the dual update rule is given by
\begin{align}
{\lamb _{n,i}^{k+1}}= {\lamb _{n,i}^{k}}+\rho h_{n,i}^{k+1}(\boldsymbol{\theta}_{n,i}^{k+1} - \boldsymbol{\Theta}_{i}^{k+1})-\rho {\bf Re}\{\bm{z}_{n,i}^{k+1}\}.
\label{dualUpdate1Noisy}
\end{align}
The aforementioned operations of A-FADMM are summarized in Algorithm~\ref{algo1} in the previous page.

\subsection{Power Control} 
Another practical concern is each worker's transmit power limitation. In order not to violate the maximum power budget $P$, before transmission each worker calculates its local power scaling factor {\sm$\alpha_{n}^{k+1}$\nm} such that ({\sm$\alpha_n^{k+1})^2\sum_{i=1}^d |{h_{n,i}^{k+1}}^*\bm{\theta}_{n,i}^{k+1} + {\bm{\lambda}_{n,i}^k}^*/\rho|^2 = P$\nm}, and sends {\sm$\alpha_{n}^{k+1}$\nm} to~the PS. Then, the PS determines {\sm$\alpha^{k+1} = \min\{\alpha_1^{k+1},\cdots, \alpha_N^{k+1}\}$\nm} that is downloaded by every worker. Finally, each worker transmits {\sm$\alpha^{k+1}( {h_{n,i}^{k+1}}^*\bm{\theta}_{n,i}^{k+1} + {\bm{\lambda}_{n,i}^k}^*/\rho)$\nm} to the PS, and after matched filtering and dividing by~{\sm$\alpha^{k+1}$\nm}, the PS obtains {\sm$\sum_{n=1}^N\!( |h_{n,i}^{k+1}|^2{\boldsymbol{\theta}}_{n,i}^{k+1} +  h_{n,i}^{k+1}{\bblambda_{n,i}^k}^*\!/\rho) + {\bm{z}}_{n,i}^{k+1}/\alpha^{k+1}$\nm} for the global model update. Note that {\sm$\alpha^{k+1}$\nm} and {\sm$\alpha_n^{k+1}$\nm} are scalar values that can be exchanged with negligible communication overhead, e.g., through separate control signaling channels \cite{3GPP:Rel15}.


\section{Convergence Analysis}
In this section, we prove the optimality and convergence of A-FADMM for convex functions under noise-free but time-varying channels. The necessary and sufficient optimality conditions are the primal and dual feasibility given by
\begin{align}
&\boldsymbol{\theta}_n^\star = \boldsymbol{\Theta}^\star\; \forall n \quad \text{and}\\
&\boldsymbol{0} \in \partial_i f_n(\boldsymbol{\theta}_n^\star)+ \boldsymbol{\mu}_{n,i}^\star\; \forall n,
\label{eq2}
\end{align}
where the superscript $\star$ denotes the value at the convergence point. The term $\bm{\mu}_{n,i}={\lamb _{n,i}}^* h_{n,i}$ is the dual variable combined with channel fading. According to \eqref{dualUpdate0}, $\bm{\mu}_{n,i}$ is updated as follows
\begin{align}
\bm{\mu}_{n,i}^{k+1} = \bm{\mu}_{n,i}^{k} + \rho |h_{n,i}^{k+1}|^2 \boldsymbol{r}_{n,i}^{k+1}, \label{Eq:mu}
\end{align}
where $\boldsymbol{r}_{n,i}^{k+1}= \boldsymbol{\theta}_{n,i}^{k+1}-\boldsymbol{\Theta}_i^{k+1}$ is the $n$-th worker's primal residual. Applying the modified dual update rule \eqref{Eq:mu} to the primal update rule \eqref{workerUpdate0}, we obtain
\begin{align}
\boldsymbol{0} \in \partial_i f_n(\boldsymbol{\theta}_{n}^{k+1}) + \bm{\mu}_{n,i}^{k+1}  + \boldsymbol{S}_{n,i}^{k+1},
\label{eq_10}
\end{align}
where $\boldsymbol{S}_{n,i}^{k+1}=\rho |h_{n,i}^{k+1}|^2 (\boldsymbol{\Theta}_i^{k+1}-\boldsymbol{\Theta}_i^{k})$ is the $n$-th worker's dual residual. Now, we are in position to introduce our first result, Lemma \ref{lemma:first}. 

\begin{lemma}\label{lemma:first}
\emph{For the iterates $\boldsymbol{\theta}_{n}^{k+1}$ the optimality gap of A-FADMM, is upper and lower bounded as follows.}
\begin{align}\label{lem1Lower}
    \text{(Lower bound)}\quad &\sum_{n=1}^N \left[f_n(\boldsymbol{\theta}_n^{k+1}) - f_n(\boldsymbol{\theta}_n^\star)\right] \nonumber \\ 
    &\geq \sum_{n=1}^N \sum_{i=1}^d \bm{\mu}_{n,i}^\star \bbr_{n,i}^{k+1}\\
\label{lem1Upper}
\nonumber \text{(Upper bound)}\quad &\sum_{n=1}^N \left[f_n(\boldsymbol{\theta}_n^{k+1}) - f_n(\boldsymbol{\theta}_n^\star)\right]\\
&\hspace{-50pt} \leq - \sum_{n=1}^N \sum_{i=1}^d \left[ \bm{\mu}_{n,i}^{k+1} \boldsymbol{r}_{n,i}^{k+1} -  \mathbf{S}_{n,i}^{k+1}(\boldsymbol{\theta}_{n,i}^\star-\boldsymbol{\theta}_{n,i}^{k+1})\right].
\end{align}
\end{lemma}
The detailed proof is provided in Appendix \ref{sec:lem1}.  The main idea for the proof is to utilize the optimality of the updates in \eqref{workerUpdate0} and \eqref{centralUpdate}. We derive the upper bound for the objective function optimality gap in terms of the primal and dual residuals as stated in \eqref{lem1Upper}. To get the lower bound in \eqref{lem1Lower} in terms of the primal residual, the definition of the Lagrangian \eqref{augmentedLagAG} is used at $\rho=0$.  The result in Lemma \ref{lemma:first} is used to derive the main results in  Theorem \ref{theorem} as presented next.

\begin{theorem}\label{theorem}
When $f_n(\boldsymbol{\theta}_n)$ is closed, proper, and convex $\forall n$ and the Lagrangian $\boldsymbol{\mathcal{L}}_{0}$ has a saddle point, under a \emph{time-varying channel}, A-FADMM satisfies the following statements. Then, the optimality gap is non-increasing, \ie 
\begin{align}
    \nonumber &\sum_{n=1}^N \sum_{i=1}^d \{ \frac{1}{\rho |h_{n,i}^{k+1}|^2} \left[ (\bm{\mu}_{n,i}^{k+1} - \bm{\mu}_{n,i}^\star)^2 -  (\bm{\mu}_{n,i}^{k} - \bm{\mu}_{n,i}^\star)^2 \right] \\
    & +  \rho |h_{n,i}^{k+1}|^2 \left[ (\boldsymbol{\Theta}_i^{k+1}- \boldsymbol{\Theta}_{i}^\star)^2 - (\boldsymbol{\Theta}_i^{k} - \boldsymbol{\Theta}_{i}^\star)^2 \right]\} \leq 0.
    \end{align}
\end{theorem}
The detailed proof of Theorem \ref{theorem} is provided in Appendix \ref{sec:them1}. For the time-invariant scenario, we have the following corollary.
\begin{corollary}\label{corollary}
    For A-FADMM under a \emph{time-invariant channel} where $h_{n,i}^{k+1}=h_{n,i}^{k}\; \forall k$, it holds that
    \begin{itemize}[leftmargin=*]
    \item The optimality gap converges to zero as  $k\rightarrow\infty$, \ie 
    \begin{align}
    \lim\limits_{k\rightarrow\infty} \sum_{n=1}^Nf_n(\boldsymbol{\theta}_n^{k})= \sum_{n=1}^Nf_n(\boldsymbol{\theta}^\star)
    \end{align}
    \item Both primal and dual residuals converge to zero as $k\rightarrow\infty$, \ie
    \begin{align}
    \lim\limits_{k\rightarrow\infty} \bbr_{n,i}^{k} =\lim\limits_{k\rightarrow\infty} \boldsymbol{S}_{n,i}^{k}=\boldsymbol{0}
    \end{align}
\end{itemize}
\end{corollary}
The proof can be found in Appendix \ref{sec:corr1}. The key idea is to define the Lyapunov function {\sm$V^k = \sum_{n=1}^N \sum_{i=1}^d [\frac{1}{\rho|h_{n,i}|^2}(\bm{\mu}_{n,i}^{k}-\bm{\mu}_{n,i}^\star)^2 + \rho  |h_{n,i}|^2 (\boldsymbol{\Theta}_i^{k} -  \boldsymbol{\Theta}_i^\star)^2]$\nm}\! and show that the difference between $V^{k+1}$ and~$V^k$ monotonically decreases with $k$. This property enables proving that both primal and dual residuals converge to zero. Next, we apply Lemma \ref{lemma:first}, and prove the optimality gap goes to~zero. 

\section{Privacy Analysis}\label{PrivacyAnalysis}
Revealing the local model update trajectory is vulnerable to model inversion and reconstruction attacks \cite{Matt:CCS15,Hitaj:CCS17}. These attacks infer the statistical profiles of training samples, violating data privacy. Against such an adversarial inverse problem, we aim to preserve privacy defined as follows.


{\bf Definition 1 \cite{zhang2018admm}}\quad A mechanism $M: M(X) \rightarrow Y$ is defined to be privacy preserving if the input~$X$ cannot be uniquely derived from the output $Y$.

We treat $X$ as local models to be protected, and consider $Y$ as the known information at an eavesdropper such as PS or another worker. Under digital transmissions, PS receives every local model $\bm{\theta}_{n}^{k+1}$, always violating privacy. Under analog FL wherein PS receives $\sum_{n=1}^N \bm{\theta}_{n}^{k+1}$ after channel inversion, for certain iterations when only one worker sends the local model to PS, privacy is violated. 

In sharp contrast, PS in A-FADMM receives $\sum_{n=1}^N( |h_{n,i}^{k+1}|^2{\boldsymbol{\theta}}_{n,i}^{k+1} +  h_{n,i}^{k+1}{\bblambda_{n,i}^k}^*/\rho)$. This does not violate privacy since the reception is the aggregate of fading-perturbed and dual-variable-distorted local models while $h_{n,i}^*$, ${\bm{\lambda}_{n,i}^k}^*$, and $N$ are unknown at~PS. Furthermore, against any eavesdropper knowing the global model trajectory, A-FADMM preserves privacy of local model and gradient trajectories as stated in the following theorems.


\begin{theorem}\label{privacyAnalysisTheorem1}
Unless $\boldsymbol{\Theta}_{i}^{k+1}=\boldsymbol{\theta}_{n,i}^{k+1}$ (i.e., before convergence), at every $k+1$, A-FADMM preserves the privacy of 
 each local model update $\boldsymbol{\theta}_{n,i}^{k+1}$ and gradient update $\partial f_n(\boldsymbol{\theta}_n^{k+1}) \;\forall n,i$.
\end{theorem}
\begin{proof}
Intuitively, we show that the inverse problem of an eavesdropper is to solve a set of equations at every iteration, in which the number of unknowns is larger than the number of equations. Therefore, each worker's local model or gradient cannot be uniquely derived. In fact, since $\boldsymbol{\theta}_{n,i}^{0}$ and $\bblambda_{n,i}^0 \forall n \in \{1,\cdots,N\}$, $\forall i \in \{1,\cdots,d\}$, are initiated randomly, then their values cannot be revealed by the eavesdropper. For simplicity, we assume that $f_n(\boldsymbol{\theta}_n^{k+1})$ is differentiable and the system is noise free. 
The eavesdropper needs to solve either of the following two equations to derive $\boldsymbol{\theta}_{n,i}^{1}$

\begin{equation}\label{inv1}
\left\{\begin{array}{l}
 \boldsymbol{\theta}_{n,i}^{1}=\frac{\rho |h_{n,i}^{1}|^2 \boldsymbol{\Theta}_i^0-\nabla_i f_n(\boldsymbol{\theta}_n^{1}) - (\lamb _{n,i}^0)^{*} h_{n,i}^{1}}{\rho |h_{n,i}^{1}|^2}, \text{ if $h_{n,i}^{1}=h_{n,i}^{0}$}\\
\boldsymbol{\theta}_{n,i}^{1}=\boldsymbol{\theta}_{n,i}^{0},  \text{ if $h_{n,i}^{1}\neq h_{n,i}^{0}$}\\
\end{array}\right.
\end{equation}

Then, we can write
\begin{align}
\nonumber \boldsymbol{\theta}_{n,i}^{1} &=\Big( \rho \sum_{n=1}^N |h_{n,i}^{1}|^2 \boldsymbol{\Theta}_i^{1}-\rho \sum_{m=1, m \neq n}^N |h_{m,i}^{1}|^2{\boldsymbol{\theta}}_{m,i}^{1} \\
& -\sum_{n=1}^N(\bblambda_{n,i}^0)^* h_{n,i}^{1} \Big) / \Big(\rho |h_{n,i}^{1}|^2 \Big).
\label{modelInversion1}
\end{align}

Note that the eavesdropper knows $\boldsymbol{\Theta}_i^0$ and $\boldsymbol{\Theta}_i^1$. However, the values of $h_{n,i}^1$, $\lamb _{n,i}^0$, $\nabla_i  f_n(\boldsymbol{\theta}_n^{1})$ , $\rho \sum_{m=1, m \neq n}^N |h_{m,i}^{1}|^2{\boldsymbol{\theta}}_{m,i}^{1}$, and $\boldsymbol{\theta}_{n,i}^{0}$ are unknown. Hence, even at the absence of the noise at the receiver, the eavesdropper cannot have a unique solution for $\boldsymbol{\theta}_{n,i}^{1}$ and/or  $\nabla_i  f_n(\boldsymbol{\theta}_n^{1})$ since the number of variables $V=5$ is greater than the number of equations $E=2$. Writing the same equations for iteration $k$ 
\begin{equation}\label{inv1}
\left\{\begin{array}{l}
 \boldsymbol{\theta}_{n,i}^{k}=\frac{\rho |h_{n,i}^{k}|^2 \boldsymbol{\Theta}_i^{k-1}-\nabla_i f_n(\boldsymbol{\theta}_n^{k}) - (\lamb _{n,i}^{k-1})^{*} h_{n,i}^{k}}{\rho |h_{n,i}^{k}|^2}, \text{ if $h_{n,i}^{k}=h_{n,i}^{k-1}$}\\
\boldsymbol{\theta}_{n,i}^{k}=\boldsymbol{\theta}_{n,i}^{k-1},  \text{ if $h_{n,i}^{k}\neq h_{n,i}^{k-1}$}\\
\end{array}\right.
\end{equation}

\begin{align}
\nonumber \boldsymbol{\theta}_{n,i}^{k} &=\Big(\rho \sum_{n=1}^N |h_{n,i}^{k}|^2 \boldsymbol{\Theta}_i^{k}-\rho \sum_{m=1, m \neq n}^N |h_{m,i}^{k}|^2{\boldsymbol{\theta}}_{m,i}^{k}  \\
& - \sum_{n=1}^N(\bblambda_{n,i}^{k-1})^* h_{n,i}^{k}\Big)/ \Big(\rho |h_{n,i}^{k}|^2\Big).
\label{modelInversionk}
\end{align}
as well as for iteration $k+1$
\small
\begin{equation}\label{inv2}
\left\{\begin{array}{l}
 \boldsymbol{\theta}_{n,i}^{k+1}= \frac{\rho |h_{n,i}^{k+1}|^2 \boldsymbol{\Theta}_i^k-\nabla_i f_n(\boldsymbol{\theta}_n^{k+1}) - (\lamb _{n,i}^k)^{*} h_{n,i}^{k+1}}{\rho |h_{n,i}^{k+1}|^2}, \text{ if $h_{n,i}^{k+1}=h_{n,i}^{k}$}\\
\boldsymbol{\theta}_{n,i}^{k+1}=\boldsymbol{\theta}_{n,i}^{k},  \text{ if $h_{n,i}^{k+1}\neq h_{n,i}^{k}$}\\
\end{array}\right.
\end{equation}
\normalsize

\begin{align}
\nonumber \boldsymbol{\theta}_{n,i}^{k+1} &=\Big( \rho \sum_{n=1}^N |h_{n,i}^{k+1}|^2 \boldsymbol{\Theta}_i^{k+1}-\rho \sum_{m=1, m \neq n}^N |h_{m,i}^{k+1}|^2{\boldsymbol{\theta}}_{m,i}^{k+1}  \\
& - \sum_{n=1}^N(\bblambda_{n,i}^k)^* h_{n,i}^{k+1}\Big)/ \Big(\rho |h_{n,i}^{k+1}|^2\Big).
\label{modelInversionk1}
\end{align}
Similarly, the eavesdropper knows  $\boldsymbol{\Theta}_i^k$ and $\boldsymbol{\Theta}_i^{k+1}$. However, $h_{n,i}^{k+1}$, $\lamb _{n,i}^k$, $\nabla_i  f_n(\boldsymbol{\theta}_n^{k+1})$ , $\rho \sum_{m=1, m \neq n}^N |h_{m,i}^{k+1}|^2{\boldsymbol{\theta}}_{m,i}^{k+1}$, and $\boldsymbol{\theta}_{n,i}^{k}$ are unknown.
We clearly see that if the algorithm has not converged to the optimal solution yet  at iteration $k+1$. i.e., $\boldsymbol{\theta}_{n,i}^{k+1}\neq\boldsymbol{\Theta}_{i}^{k+1}$, then there is no unique inversion of $\boldsymbol{\theta}_{n,i}^{k+1}$ since the number of variables is more than the number of equations. This finalizes the proof.
\end{proof}

\begin{theorem}\label{privacyAnalysisTheorem2}
When  $\boldsymbol{\Theta}_{i}^{k+1}=\boldsymbol{\theta}_{n,i}^{k+1}$ (i.e., at convergence), A-FADMM preserves the privacy of the local model trajectory $\{\boldsymbol{\theta}_{n,i}^0,\cdots, \boldsymbol{\theta}_{n,i}^{k}\}$ and gradient trajectory $\{\partial f_n(\boldsymbol{\theta}_n^{1}), \cdots, \partial f_n(\boldsymbol{\theta}_n^{k+1})\} \;\forall n,i$.
\end{theorem}
\begin{proof}
In brief, we show that after A-FADMM convergence when all local models become identical and known to an eavesdropper, this information cannot be used to derive a unique trajectory of each worker's local model and gradient updates. When $\boldsymbol{\theta}_{n,i}^{k+1}=\boldsymbol{\Theta}_{i}^{k+1}, \forall n,i$, we know from \eqref{modelInversionk1} that the following terms can be found at the PS: $ \sum_{n=1}^N |h_{n,i}^{k+1}|^2 \boldsymbol{\Theta}_i^{k+1}$, $\rho \sum_{m=1, m \neq n}^N |h_{m,i}^{k+1}|^2{\boldsymbol{\theta}}_{m,i}^{k+1}$. However, the terms $(\bblambda_{n,i}^k)^* h_{n,i}^{k+1}$ and $\rho |h_{n,i}^{k+1}|^2$ cannot be found, and these two terms are needed to retrieve a unique solution for $\nabla_i  f_n(\boldsymbol{\theta}_n^{k+1})$ using \eqref{inv2}. Hence, $\nabla_i  f_n(\boldsymbol{\theta}_n^{k+1})$ cannot be uniquely derived. From \eqref{inv1}-\eqref{modelInversionk}, we clearly see that knowing $\boldsymbol{\theta}_{n,i}^{k+1}$, $\boldsymbol{\Theta}_{i}^{k}$, and $\boldsymbol{\Theta}_{i}^{k-1}$  are not enough to find a unique solution for  $\nabla_i  f_n(\boldsymbol{\theta}_n^{k})$ and $\boldsymbol{\theta}_{n,i}^{k}$ since all other terms in the two equations including $h_{n,i}^{k}$ and $\bblambda_{n,i}^{k-1}$ are also unknown. Therefore, the individual model at the convergence point do not release any unique information about the updating steps of the model and the function gradient trajectory, which concludes the proof.
\end{proof}

\section{Experiments}\label{Sec:Exp}
To validate our theoretical foundations, we numerically evaluate the performance of A-FADMM in convex (linear regression) and non-convex (image classification using DNNs) problems. 

\subsection{Simulation settings} 
For linear regression, we use the California Housing dataset \cite{Torgo:14} consisting of $20000$ samples with $6$ features, i.e., model size $d=6$. 
At iteration $k$, the loss is given as {\sm$|\sum_{n=1}^N [f(\boldsymbol{\theta}_{n}^k)-f(\boldsymbol{\theta}^\star)]|$\nm}. For image classification, we use the MNIST dataset~\cite{LeCun:MNiST} comprising $60000$ training and $10000$ test samples, each of which represents a hand-written $0$-$9$ digit image. In this case, we consider a $3$-layer fully connected multi-layer perceptron (MLP) comprising an input layer with $784$ neurons, two hidden layers with $128$ and $64$ neurons, respectively, and an output layer with $10$ neurons, resulting in the model size $d=109184$. We use the rectified linear unit (ReLu) activation function, softmax outuput, and cross entropy loss.


By default, we consider $N\!=\!100$ workers with $\text{SNR}\!=\! 40$dB, each of which stores the same number of training samples equally divided and allocated from the training dataset. These workers are supported using $10$ and $4096$ subcarriers for linear regression and DNNs, respectively. Following the LTE cellular standards~\cite{3GPP:Rel15}, each subcarrier provides $15$KHz bandwidth during $1$ms. Each channel realization is coherent during $10$ iterations, and is randomly generated by a Rayleigh fading distribution with zero mean and unit variance for every $10$ iterations. 

To focus primarily on the uplink bandwidth bottleneck in the simulations, analog transmissions are utilized only for the uplink, while digital transmissions are considered in the downlink where the PS broadcasting the global updates without any bandwidth competition. Consequently, in the resultant A-FADMM implementation under channel noise, the global model update after the analog uplink reception follows \eqref{centralUpdateNoisy} as in Algorithm~\ref{algo1}, whereas the primal and dual updates (originally given as \eqref{workerUpdateNoisy} and \eqref{dualUpdate1Noisy} in Algorithm~\ref{algo1}) after the digital downlink reception use the following rules:
\begin{align}
&\hspace{-5pt}\boldsymbol{0} \in \partial_i f_n(\boldsymbol{\theta}_n^{k+1}) + (\lamb _{n,i}^k)^*h_{n,i}^{k+1} +\rho |h_{n,i}^{k+1}|^2( \boldsymbol{\theta}_{n,i}^{k+1} \!-\! \boldsymbol{\Theta}_i^k)  \label{App:Primal_nf} \\
&\hspace{-5pt} {\lamb _{n,i}^{k+1}}= {\lamb _{n,i}^{k}}+\rho h_{n,i}^{k+1}(\boldsymbol{\theta}_{n,i}^{k+1} \!-\! \boldsymbol{\Theta}_{i}^{k+1}). \label{App:Dual_nf}
    \end{align}
These noise-free primal and dual update rules are implemented as follows. In the digital downlink, each worker decodes $\bm{\Theta}_i^{k+1}$, and manually perturbs it as $|h_{n,i}^{k+1}|^2\bm{\Theta}_i^{k+1}$ that is used for updating primal and dual variables via \eqref{App:Primal_nf} and \eqref{App:Dual_nf}.

In A-FADMM, the $i$-th element of the models of all workers are uploaded using the $i$-th sub-carrier. In linear regression, the model size is less than the number of available subcarriers, i.e., $d=6 < 10$, and hence A-FADMM requires only one time slot (one upload) to upload all workers' models at each iteration. In image classification where $d=109184$, it requires $\lceil{109184}/{4096}\rceil=27$ time slots to uploads all workers' models per iteration.


In D-FADMM, the number of uploading time slots depends not only on the number of subcarriers but also on the channel gain of each subcarrier. To be precise, following the LTE cellular standards~\cite{3GPP:Rel15}, each subcarrier provides $W_i=15$KHz bandwidth during $1$ms. Each channel realization is coherent during $10$ iterations, and is randomly generated by a Rayleigh fading distribution with zero mean and unit variance for every $10$ iterations. When each model element consumes $32$ bits, the $n$-th worker requires the uploading time slots $\hat{T}_n$ that is the minimum $T_n$ satisfying the following condition $\int_{t=1}^{T_n}\sum_{i=1}^{4096\!/\!N}\! R_{n,i}(t)\text{dt} \geq 32 d$, where $R_{n,i}(t) = W_i\log_2(1 + P|h_{n,i}(t)|/(W_iN_0))$ follows from the Shannon formula. Since each worker has independent channel realizations, to upload all workers' models to PS, it requires $\hat{T}=\max\{\hat{T}_1,\hat{T}_2,\cdots,\hat{T}_N\}$ time slots. Based on Algorithm~\ref{algo1}, D-FADMM is implemented by replacing \eqref{centralUpdateNoisy} for $\boldsymbol{\Theta}_i^{k+1}$,
\eqref{workerUpdateNoisy} for $\boldsymbol{\theta}_{n,i}^{k+1}$, and \eqref{dualUpdate1Noisy} for $\lamb _{n,i}^{k+1}$ with \eqref{workerUpdate0}, \eqref{centralUpdate}, and \eqref{dualUpdate0}, respectively.

In \eqref{augmentedLagAG}, we choose the penalty constant $\rho=0.5$ yielding fast convergence for both digital and analog implementations from our observations. To run the experiments, we use Matlab for linear regression and TensorFlow for image classification, operated in a MacBook Air computer (1.8 GHz Intel Core i5 CPU, 8 GB 1,600 MHz DDR3 RAM). For each plot, we run $5$ simulations, and report mean values (solid curves) and standard deviations (shaded areas, omitted for negligible values). Finally, we compare A-FADMM with the following benchmark algorithms.


 \begin{itemize}[leftmargin=*]
    \item \textbf{D-FADMM} is the digital communication version of A-FADMM, wherein the total bandwidth is equally divided and allocated to each worker whose model element consumes $32$ bits. i.e., the value of each element in the model vector is transmitted using 32 bits. Following A-FADMM, we use $\rho=0.5$.
    
    \item \textbf{A-GD} is the analog communication versions of the distributed gradient descent algorithm~(GD) with channel inversion that allows the $n$-th worker to upload its update only when the channel gain $|h_{n,i}^k|\geq \epsilon$. We use $\epsilon=10^{-6}$ for communication and $10^{-4}$ learning rate for GD operations. We observed that A-GD diverges for a larger learning rate. \vspace{3pt}
 \end{itemize}
For the image classification task, we use the following baselines and hyperparameters.
\begin{itemize}[leftmargin=*]
    \item \textbf{A-SFADMM} is the stochastic version of A-FADMM using DNNs. For the ADMM problem, we use $\rho=0.5$. For the local problem at each global iteration, each worker selects a mini-batch of size $100$ samples at random, and uses the Adam optimizer with $0.01$ learning rate to update its local model. Per global iteration, we consider $20$ local iterations. For different choices of local iterations and learning rates, we study their impact on convergence speed and accuracy in Figures \ref{localIterFigDnn} and \ref{lrFigDNN}.
    
    \item \textbf{D-SFADMM} is the stochastic version of D-FADMM which are utilized in the classification problem using DNN. Following A-SFAMM, we use $\rho=0.5$, Adam optimizer with $0.01$ learning rate, mini-bath size $100$, and $20$ local iterations per global iteration.

    \item \textbf{A-SGD} is the stochastic version of A-GD with channel inversion (i.e., analog FL). Following A-GD, we use $\epsilon=10^{-6}$ for communication. For SGD operations, we use mini-batch size $100$, and choose the learning rate $0.005$. Note that from our observations, A-SGD incurs high oscillation under the learning rate $0.01$ used in A-SFADMM and D-SAFDMM.
 \end{itemize}

For both linear regression and image classification task,   
the notation \textbf{10x} implies an algorithm with 10x more subcarriers (bandwidth) than the default setting. For example, compared to A-SFADMM using $4096$ subcarriers, A-SFADMM-10x utilizes $40960$ subcarriers at each iteration. Accordingly, given the MLP model size $d=109184$, A-SFADMM-10x requires $\lceil{109184}/{40960}\rceil=3$ time slots for uploading all workers' models, which is $9$x less than A-SFADMM requiring $\lceil{109184}/{4096}\rceil=27$ time slots.

\begin{figure*}[t]\label{eval}
\centering
\includegraphics[width=\textwidth]{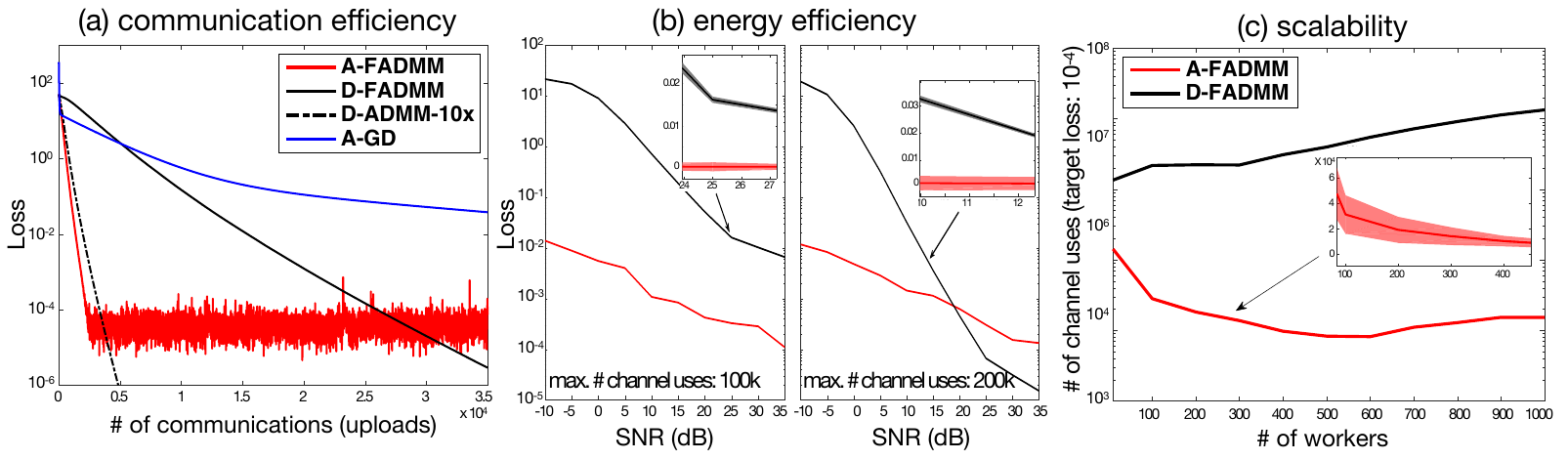}
\caption{\emph{Linear regression} results showing: (a) communication efficiency (loss w.r.t. \# of uploads); (b) energy efficiency (loss w.r.t. SNR); and (c) scalability (\# of channel uses w.r.t \# of workers).}
\label{Fig_LR} 
\end{figure*}

\begin{figure*}[t]
    \centering
    \includegraphics[width=\textwidth]{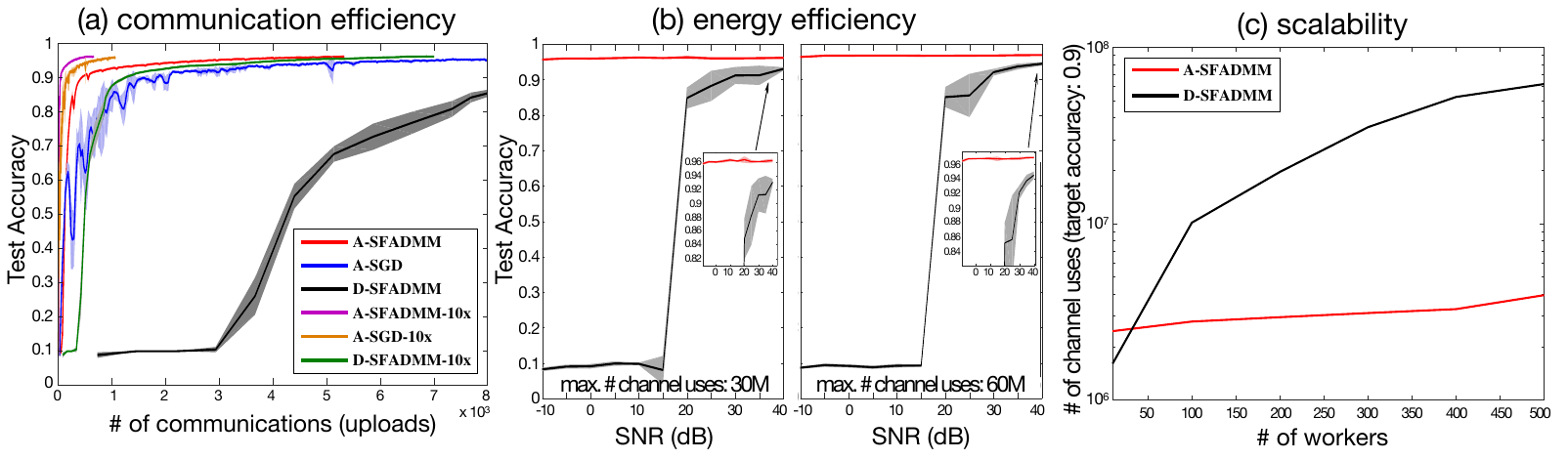}
    \caption{\emph{Image classification} results showing: (a) communication efficiency (test accuracy w.r.t. \#~of uploads); (b)~energy efficiency (test accuracy w.r.t.~SNR); and (c) scalability (\# of channel uses w.r.t \# of workers).}
    \label{Fig_DNN} 
    \end{figure*}

\subsection{Communication Efficiency}
In linear regression, as observed in Fig.~\ref{Fig_LR}(a), A-FADMM requires the lowest communication rounds until achieving a target loss $10^{-4}$. Even with $10$x more subcarriers, D-FADMM fails to reach the same speed due to the orthogonal subcarrier allocation to each worker under limited bandwidth. However, if one aims to achieve very low loss below $10^{-4}$, A-FADMM suffers from noisy reception, and D-FADMM may thus be a better choice, as long as very large bandwidth and/or long uploading time are available. In image classification, Fig.\ref{Fig_DNN}(a) shows that A-SFADMM achieves the highest accuracy the minimum number of communication rounds. In fact, it is even more communication-efficient than D-SFADMM with $10$x more subcarriers (D-SFADMM-10x).

For both tasks, analog FL (i.e., A-GD and A-SGD) struggles with intermittent uploads due to the truncated channel inversion (transmitting only when {\sm$|h_{n,i}|\geq \varepsilon$\nm}). This yields too many communication rounds in linear regression (A-GD) and high variance in image classification (A-SGD), highlighting the importance of non-channel inversion methods used in A-FADMM and A-SFADMM.

\subsection{Energy Efficiency vs. Accuracy}
In this experiment, we assume that there are sufficient subcarriers to upload every update in one time slot, and focus on wireless communication energy consumption that often exceeds computing energy~\cite{Elbamby:19}. We measure the loss or accuracy when the total channel uses $\sum_{i=1}^j M_i$ at time slot $j$ reaches a target maximum number of channel uses, where $M_i$ is the number of subcarriers used in time slot $i$.

With linear regression task and $100$k maximum number of channel uses, Fig.~\ref{Fig_LR}(b) shows that A-FADMM always achieves order-of-magnitude lower loss than D-FADMM, even at very low $-10$dB SNR, i.e., low transmit power. With $200$k channel uses, D-FADMM outperforms A-FADMM, but only at high SNR exceeding $20$dB. This advocates that A-FADMM is more energy-efficient and bandwidth-efficient.

In image classification, as shown by Fig.~\ref{Fig_DNN}(b), A-SFADMM not only outperforms D-SFADMM, but also achieves the maximum test accuracy even when the SNR is as low as $-10$dB and the maximum number of channel uses is $30$M. By contrast, D-SFADMM with $40$dB SNR and $60$M channel uses achieves maximum accuracy that is still lower than A-SFADMM's.

\subsection{Scalability}
We investigate the scalability of A-FADMM and A-SFADMM, by counting the number of channel uses until reaching a target loss or accuracy. We vary the number of contributing workers, and we assume that the noise power spectral density is fixed as $10^{-9}$W/Hz. In linear regression, we clearly see from Fig.\ref{Fig_LR}(c) that A-FADMM does not require more channel uses for more workers to achieve a target loss $10^{-4}$. By contrast, D-FADMM necessitates the channel uses linearly proportional to the number of workers due to the orthogonal bandwidth allocation to every worker. It is worth mentioning that even with only $N=10$ workers, A-FADMM requires order of magnitude less channel uses than D-FADMM. Similar trends are observed in Fig.\ref{Fig_DNN}(b) for image classification, only except for the cases below $N=10$ workers.

\subsection{Sensitivity Analysis}
 \label{Sec:sensAnalysis}
 
In this subsection, we study the impact of hyperparameters on the convergence speed and accuracy of A-FADMM and D-FADMM as well as their stochastic versions. All the training, communication, and simulation environments are identical to the settings in Sec.~\ref{Sec:Exp}, except for the hyperparameters: disagreement penalty weight $\rho$, learning rate, and the number of local iterations as elaborated next.

\begin{figure}[h]
    \centering 
    \subfigure[Linear regression.]{\includegraphics[width=.495\textwidth]{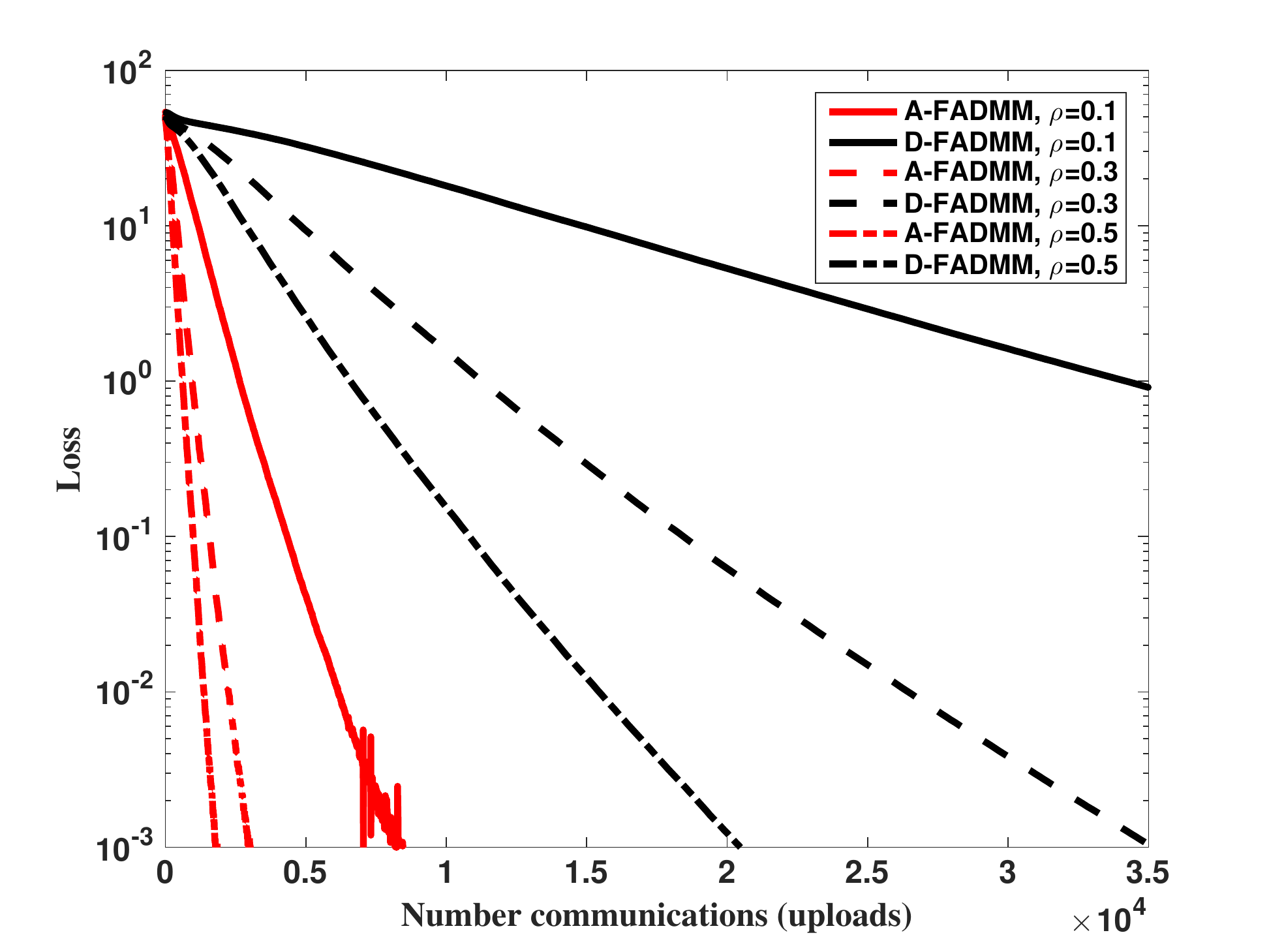}} 
    \subfigure[Image classification.]{\includegraphics[width=.495\textwidth]{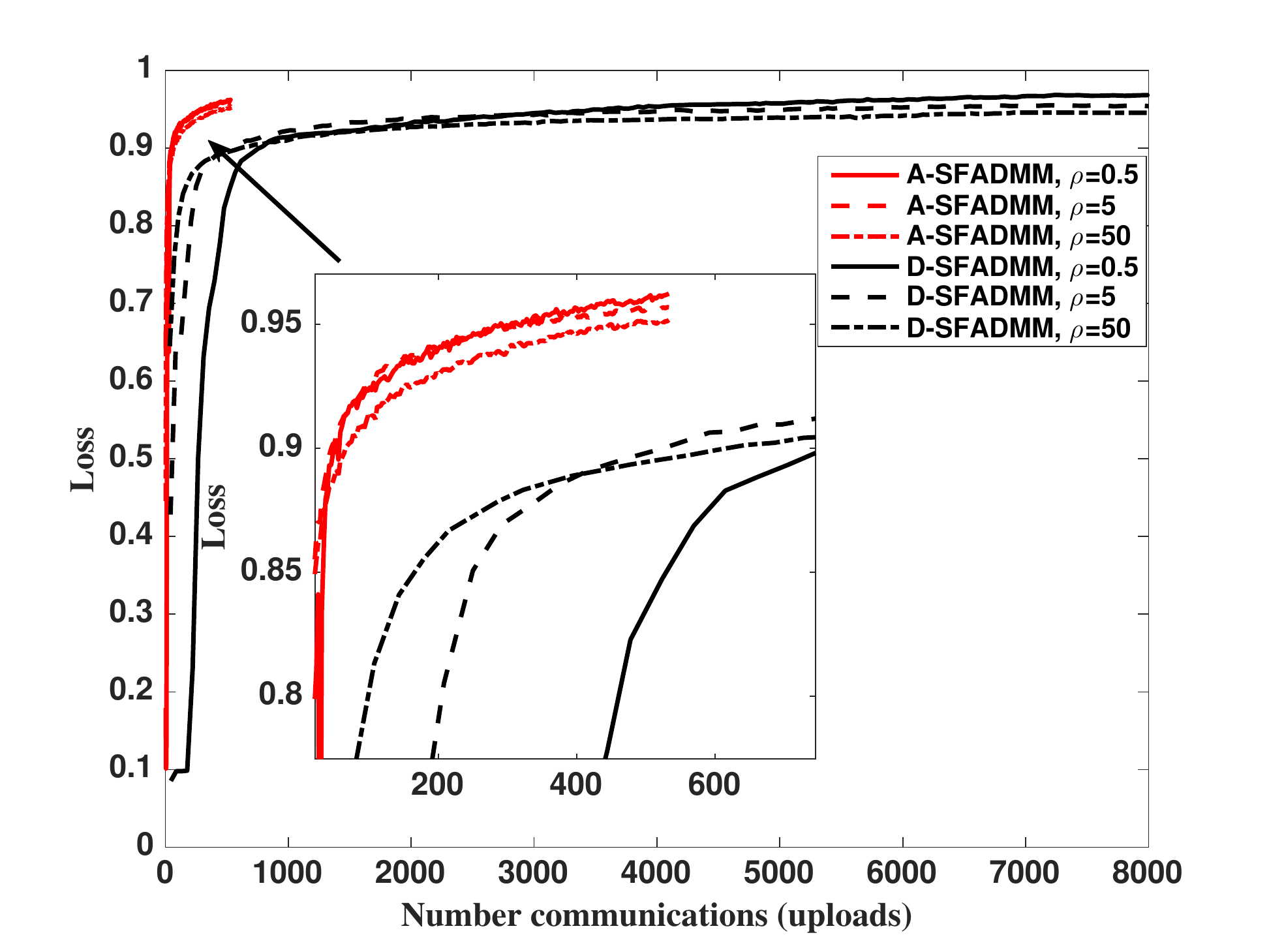}}
    \caption{Impact of the \emph{disagreement penalty weight $\rho$} in (a) linear regression and (b) image classification using DNNs.}
    \label{rhoFig}
    \end{figure}

 
 \paragraph{Impact of $\rho$} The penalty weight $\rho$ adjusts the degree of disagreement between local and global models in both linear regression and classification tasks. In linear regression, Fig.~\ref{rhoFig}(a) shows that a larger $\rho$ leads to faster convergence with diminishing returns for both A-FADMM and D-FADMM. Our choice $\rho=0.5$ in Sec.~\ref{Sec:Exp} is thus a value yielding sufficiently fast convergence.
 
 In image classification, on the other hand, Fig.~\ref{rhoFig}(b) shows that a smaller $\rho$ is slower at the beginning, but reaches the highest test accuracy faster. For small~$\rho$, the penalty of disagreeing with other workers is not large. Therefore, every worker is likely to be biased towards its local optima. Since each worker has only a fraction of the global dataset, the convergence speed is fast, but the accuracy cannot outperform the global model averaged across all workers. For large~$\rho$, workers tend to strictly reduce the local model disagreement from the beginning. This yields a faster jump to a high accuracy level at the early phase. However, keeping large $\rho$ slows down the updating step by pushing all workers towards minimizing the disagreement in their model updates at every iteration. Given these observations, our choice $\rho=0.5$ in Sec.~\ref{Sec:Exp} is a value yielding sufficiently fast convergence to the highest accuracy. To obviate the accuracy reduction at the beginning while keeping fast convergence speed, studying time-varying $\rho$ (e.g., decreasing $\rho$ from a large value with the number of iterations) could be an interesting topic for future study.


\paragraph{Impact of the Number of Local Iterations}
Our image classification relies on DNNs, and thus cannot be solved in a closed form expression. Instead, at every global iteration $k$, several local iterations are performed, updating each local model. Ideally, each worker needs to iterate until convergence before sharing the model update, which may however consume too much time. Alternatively, following the standard FL settings \cite{Brendan17,pap:jakub16,Google:FL19}, we run the local training algorithm (i.e., Adam for A-SFADMM and D-SFADMM and SGD for A-SGD) for a few iterations before uploading each model. The number of local iterations is critical in ensuring convergence and achieving high accuracy. As shown in Fig.~\ref{localIterFigDnn}, with $5$ local iterations both A-SFADMM and D-SFADMM suffer from low accuracy, while A-FADMM even struggles with oscillation. With $20$ local iterations, we observe that both A-SFADMM and D-SFADMM achieve not only convergence but also the highest accuracy. Optimizing the number of local iterations is intertwined with learning rate, mini-batch size, and communication channels. This interesting-but-challenging problem is deferred to future work.

\begin{figure}[h]
\centering
\includegraphics[width=.5\textwidth]{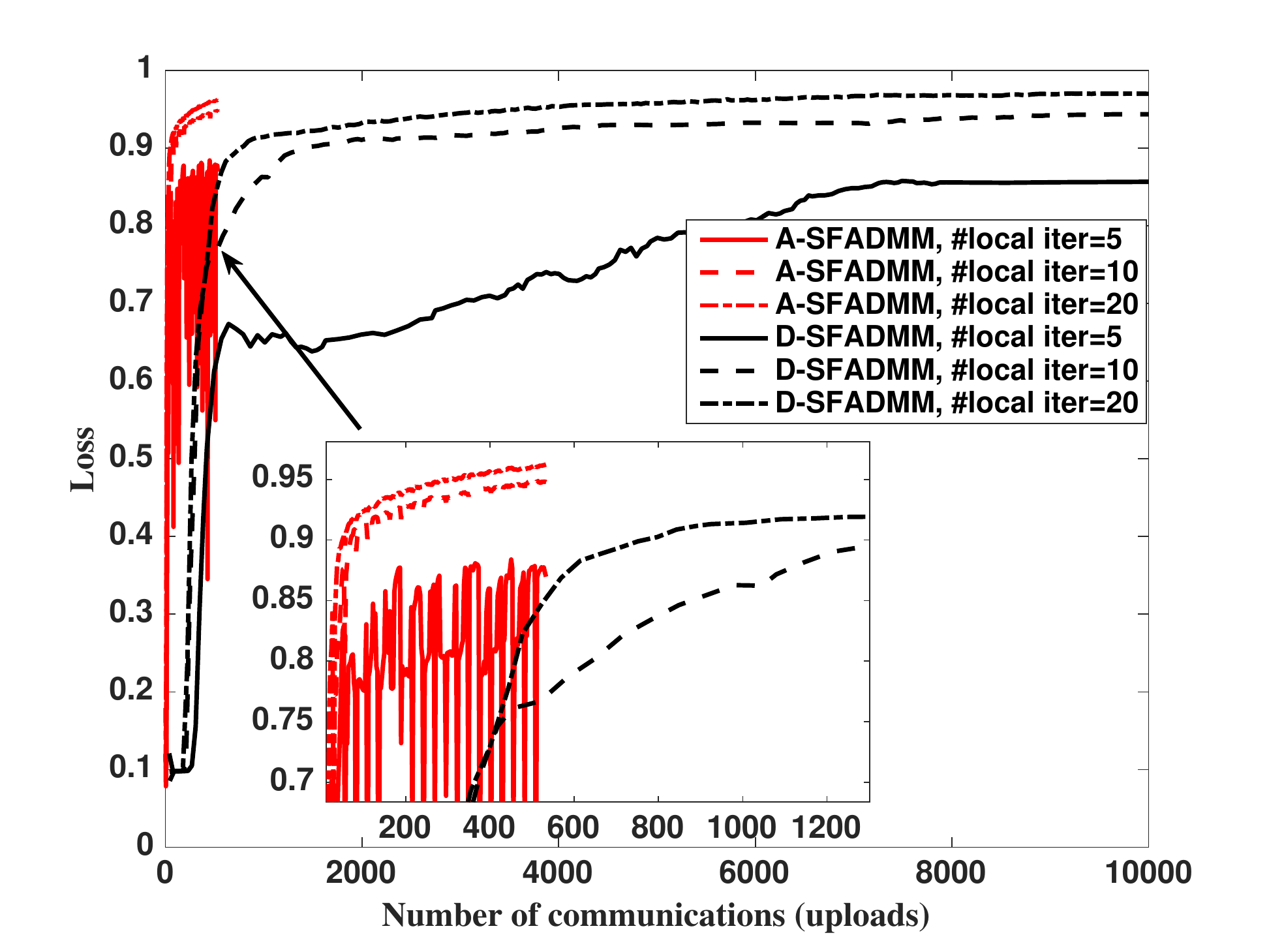}
\caption{Impact of the \emph{number of local iterations} in image classification using DNNs.}
\label{localIterFigDnn} 
\end{figure}

\paragraph{Impact of Learning Rates}
In Sec.~\ref{Sec:Exp}, we use the local optimizer's (Adam or SGD) learning rate $0.01$. Here, we additionally test the learning rate $0.001$ under $\rho=5$. As shown in Fig.~\ref{lrFigDNN}, for A-SFADMM, the learning rate change does not affect the convergence speed and accuracy significantly. By contrast, for D-SFADMM, the learning rate $0.01$ leads to faster convergence, while for A-SGD, the learning rate $0.001$ yields less oscillation.

\begin{figure}[h]
    \centering
    \includegraphics[width=.5\textwidth]{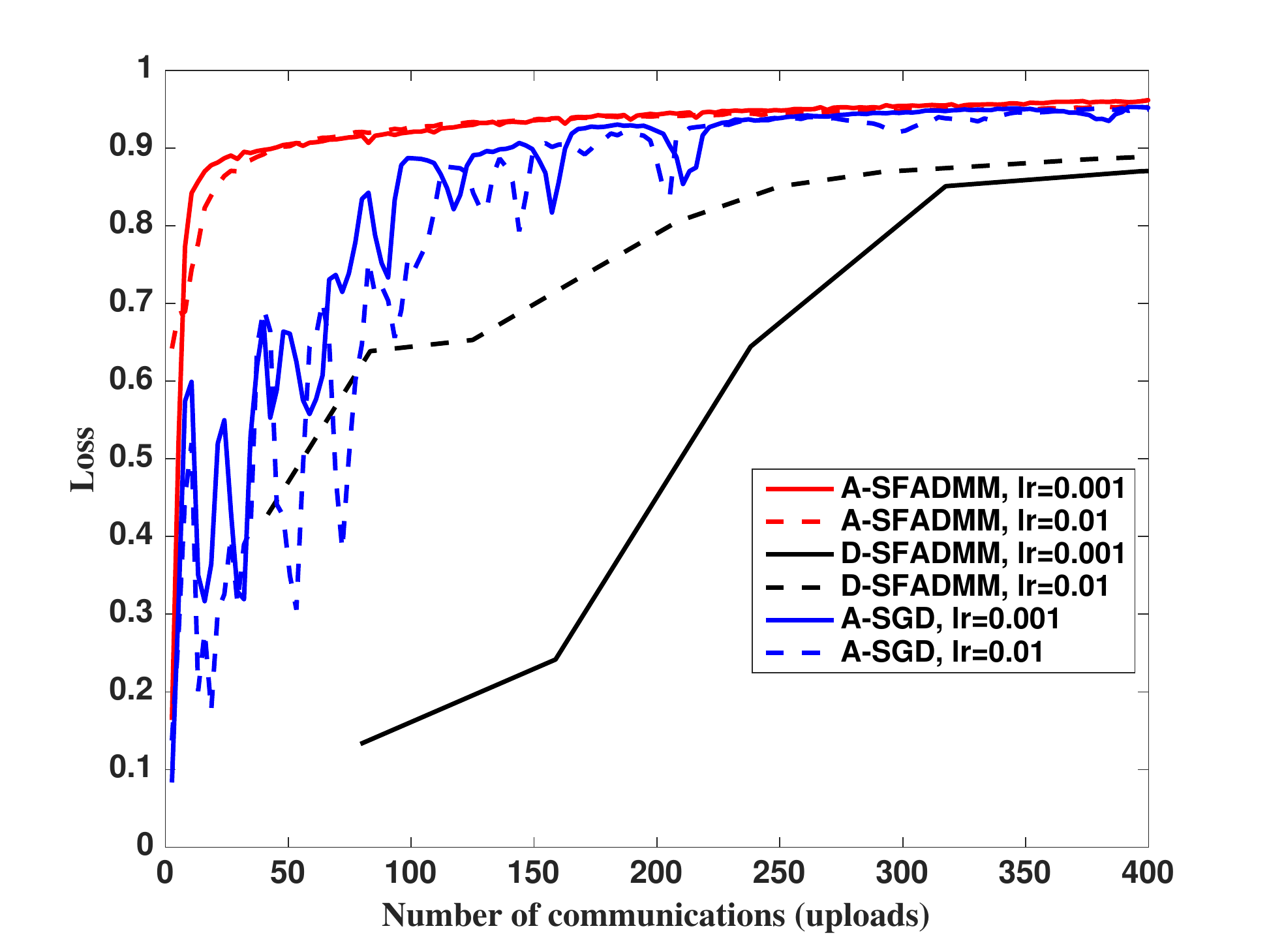}
    \caption{Impact of the \emph{learning rate} in image classification using DNNs.}
    \label{lrFigDNN} 
    \end{figure}
    


\section{Conclusion}\label{SecConc}
In this article, we proposed A-FADMM, and proved its theoretical convergence and privacy guarantees, while validating its effectiveness in convex and non-convex problems. To further improve the applicability, we conclude this article by addressing several practical issues and possible extensions.
\begin{itemize}
\item \textbf{Asynchronous Transmissions:} 
Analog over-the-air aggregation is sensitive to asynchronous signal transmissions as both early and delayed arrivals incur additional noise at reception. To alleviate this problem, it is possible to turn the less communication rounds of A-FADMM into longer transmission time of each worker, increasing the signal overlapping duration compared to the out-of-synch duration.
\item \textbf{Large Models:}
To convey large models using analog signals, model compression methods should be re-designed. Applying compressive sensing techniques is promising, in which a sparsified update is encoded by multiplying a random matrix before transmission \cite{Amiri:2019aa}, and the received update is decoded using the approximate message passing (AMP) algorithm \cite{Donoho18914}.
\item \textbf{Decentralized Architecture:}
Workers have limited transmit energy, and hence faraway workers are difficult to reach PS~\cite{elgabli2019gadmm}\cite{benissaid2020}, hindering the wide-area coverage of A-FADMM. It could be therefore interesting to study the decentralized version of A-FADMM in which every worker communicates only with neighbors while taking into account their time-varying network topologies.
\end{itemize}



\section{Appendices}

\subsection{Proof of Lemma \ref{lemma:first}}
\label{sec:lem1}

To prove the statement of the lemma, we will proceed by proving the following two statements

$(i)$ The upper bound on the optimality gap is given as
\begin{align}
\nonumber &\sum_{n=1}^N \left[f_n(\boldsymbol{\theta}_n^{k+1}) -f_n(\boldsymbol{\theta}_n^\star)\right] \\
&\leq - \sum_{n=1}^N \sum_{i=1}^d \bm{\mu}_{n,i}^{k+1} \boldsymbol{r}_{n,i}^{k+1} + \sum_{n=1}^N \sum_{i=1}^d \mathbf{S}_{n,i}^{k+1}(\boldsymbol{\theta}_{n,i}^\star-\boldsymbol{\theta}_{n,i}^{k+1}),
\end{align}
where $\bm{\mu}_{n,i}={\lamb _{n,i}}^* h_{n,i}$.

$(ii)$ The lower bound on the optimality gap is given as	
\begin{align}
&\sum_{n=1}^N [f_n(\boldsymbol{\theta}_{n}^{k+1})-f_n(\boldsymbol{\theta}^\star)] \geq -\sum_{n=1}^N \sum_{i=1}^d \bm{\mu}_{n,i}^\star \bbr_{n,i}^{k+1}.
\end{align}

\emph{Proof of statement (i):}
We note that $f_n(\boldsymbol{\theta}_n)$ for all $n$ is closed, proper, and convex, hence $\boldsymbol{\mathcal{L}}_{\rho}$ is sub-differentiable. Since $\boldsymbol{\theta}_{n}^{k+1}$ minimizes $\boldsymbol{\mathcal{L}}_{\rho}( \boldsymbol{\theta}_{n},\boldsymbol{\Theta}^{k}, \lamb_{n}^k)$, the following must hold true at each iteration $k+1$
\begin{align}
\boldsymbol{0} \in \partial_i f_n(\boldsymbol{\theta}_n^{k+1}) + \bm{\mu}_{n,i}^{k} + \rho |h_{n,i}^{k+1}|^2 \boldsymbol{\theta}_{n,i}^{k+1} -\rho |h_{n,i}^{k+1}|^2 \boldsymbol{\Theta}_i^k.
\label{eqToUpdateLambdaandTheta}
\end{align}
Note that when $h_{n,i}^{k+1} \neq h_{n,i}^k$, we choose $\boldsymbol{\theta}_{n,i}^{k+1}=\boldsymbol{\theta}_{n,i}^{k}$, and under this choice, $\boldsymbol{\theta}_{n}^{k+1}$ is still the minimizer of $\boldsymbol{\mathcal{L}}_{\rho}( \boldsymbol{\theta}_{n},\boldsymbol{\Theta}^{k}, \lamb_{n}^k)$ since $\lamb_{n,i}^k$ should have been calculated to satisfy \eqref{eqToUpdateLambdaandTheta} given $\boldsymbol{\theta}_{n,i}^{k}$ when there is change in the channel.

Adding and subtracting the term $\rho |h_{n,i}^{k+1}|^2 \boldsymbol{\Theta}_i^{k+1}$ and re-arranging the terms, we can write
\begin{align}
\nonumber \boldsymbol{0} &\in \partial_i f_n(\boldsymbol{\theta}_n^{k+1}) + \bm{\mu}_{n,i}^{k} +\rho |h_{n,i}^{k+1}|^2 \left(\boldsymbol{\theta}_{n,i}^{k+1} -\boldsymbol{\Theta}_i^{k+1}\right)\\
& + \rho |h_{n,i}^{k+1}|^2 \left(\boldsymbol{\Theta}_i^{k+1}-\boldsymbol{\Theta}_i^k\right).
\end{align}
Using the definitions of $\bm{r}_{n,i}^{k+1}$ and $\mathbf{S}_{n,i}^{k+1}$ as well as the update of $\bm{\mu}_{n,i}^{k+1}$ given in Eq. \eqref{Eq:mu}, we obtain
\begin{align}\label{first}
\boldsymbol{0} \in \partial_i f_n(\boldsymbol{\theta}_n^{k+1}) + \bm{\mu}_{n,i}^{k+1} + \mathbf{S}_{n,i}^{k+1}.
\end{align}
The result in \eqref{first} implies that $\boldsymbol{\theta}_n^{k+1}$ minimizes the following convex objective function 
\begin{align}\label{second}
f_n(\boldsymbol{\theta}_n) + \sum_{i=1}^d  \bm{\mu}_{n,i}^{k+1} \boldsymbol{\theta}_{n,i} + \sum_{i=1}^d  \mathbf{S}_{n,i}^{k+1} \boldsymbol{\theta}_{n,i}.
\end{align}
Next, since $\boldsymbol{\theta}_n^{k+1}$ is the minimizer of \eqref{second}, then, it holds that
\begin{align}\label{third}
\nonumber &f_n(\boldsymbol{\theta}_n^{k+1}) + \sum_{i=1}^d  \bm{\mu}_{n,i}^{k+1} \boldsymbol{\theta}_{n,i}^{k+1}
+  \sum_{i=1}^d  \mathbf{S}_{n,i}^{k+1} \boldsymbol{\theta}_{n,i}^{k+1}\\
&\leq f_n(\boldsymbol{\theta}_n^\star) +  \sum_{i=1}^d  \bm{\mu}_{n,i}^{k+1} \boldsymbol{\theta}_{n,i}^\star
+  \sum_{i=1}^d  \mathbf{S}_{n,i}^{k+1} \boldsymbol{\theta}_{n,i}^\star,
\end{align} 
where $\boldsymbol{\theta}^\star$ is the optimal value of the problem in \eqref{com_agadmm}-\eqref{com_agadmm_c1}. 
Summing over all workers yields
\begin{align}\label{third2}
&\!\sum_{n=1}^N f_n(\boldsymbol{\theta}_n^{k+1}) + \sum_{n=1}^N \sum_{i=1}^d  \bm{\mu}_{n,i}^{k+1} \boldsymbol{\theta}_{n,i}^{k+1}
+ \sum_{n=1}^N \sum_{i=1}^d \mathbf{S}_{n,i}^{k+1}\boldsymbol{\theta}_{n,i}^{k+1}
\nonumber\\&
\leq \sum_{n=1}^N f_n(\boldsymbol{\theta}_n^\star) + \sum_{n=1}^N \sum_{i=1}^d  \bm{\mu}_{n,i}^{k+1} \boldsymbol{\theta}_{n,i}^\star
+ \sum_{n=1}^N \sum_{i=1}^d  \mathbf{S}_{n,i}^{k+1}\boldsymbol{\theta}_{n,i}^\star
\end{align} 
Similarly, $\boldsymbol{\Theta}_{i}^{k+1}$ satisfies 
\begin{align}
0 = -\sum_{n=1}^N \bm{\mu}_{n,i}^{k} + \rho \sum_{n=1}^N |h_{n,i}^{k+1}|^2\boldsymbol{\Theta}_{i}^{k+1} -\rho \sum_{n=1}^N |h_{n,i}^{k+1}|^2 \boldsymbol{\theta}_{n,i}^{k+1}.
\end{align}
Using the update of $\bm{\mu}_{n,i}^{k+1}$, we deduce that $\boldsymbol{\Theta}_{i}^{k+1}$ minimizes $ -\sum\limits_{n=1}^N \bm{\mu}_{n,i}^{k+1} \boldsymbol{\Theta}_{i}$, and therefore, we can write
\begin{align}
-\sum_{n=1}^N \bm{\mu}_{n,i}^{k+1} \boldsymbol{\Theta}_{i}^{k+1} \leq  -\sum_{n=1}^N \bm{\mu}_{n,i}^{k+1} \boldsymbol{\Theta}_{i}^\star.
\end{align}
Summing over all $i$ yields
\begin{align}\label{fourth}
-\sum_{n=1}^N \sum_{i=1}^d \bm{\mu}_{n,i}^{k+1} \boldsymbol{\Theta}_{i}^{k+1} \leq  -\sum_{n=1}^N \sum_{i=1}^d \bm{\mu}_{n,i}^{k+1} \boldsymbol{\Theta}_{i}^\star.
\end{align}
Adding \eqref{third2} and \eqref{fourth}, we get
\begin{align}
\nonumber &\sum_{n=1}^N f_n(\boldsymbol{\theta}_n^{k+1}) + \sum_{n=1}^N \sum_{i=1}^d  \bm{\mu}_{n,i}^{k+1} \boldsymbol{\theta}_{n,i}^{k+1} + \sum_{n=1}^N \sum_{i=1}^d  \mathbf{S}_{n,i}^{k+1}\boldsymbol{\theta}_{n,i}^{k+1} \\
\nonumber & - \sum_{n=1}^N \sum_{i=1}^d \bm{\mu}_{n,i}^{k+1}\boldsymbol{\Theta}_{i}^{k+1}
\\ \nonumber &\leq \sum_{n=1}^N f_n(\boldsymbol{\theta}_n^\star) + \sum_{n=1}^N \sum_{i=1}^d \bm{\mu}_{n,i}^{k+1} \boldsymbol{\theta}_{n,i}^\star + \sum_{n=1}^N \sum_{i=1}^d  \mathbf{S}_{n,i}^{k+1}\boldsymbol{\theta}_{n,i}^\star\\
& - \sum_{n=1}^N \sum_{i=1}^d \bm{\mu}_{n,i}^{k+1} \boldsymbol{\Theta}_{i}^\star.
\end{align}
After rearranging the terms, we get
\begin{align}
\nonumber &\sum_{n=1}^N \left[f_n(\boldsymbol{\theta}_n^{k+1}) -f_n(\boldsymbol{\theta}_n^\star)\right] 
\\ \nonumber &\leq - \sum_{n=1}^N \sum_{i=1}^d \bm{\mu}_{n,i}^{k+1} (\boldsymbol{\theta}_{n,i}^{k+1} - \boldsymbol{\Theta}_{i}^{k+1}) + \sum_{n=1}^N \sum_{i=1}^d \bm{\mu}_{n,i}^{k+1}(\boldsymbol{\theta}_{n,i}^\star - \boldsymbol{\Theta}_{i}^\star) \\
&+ \sum_{n=1}^N \sum_{i=1}^d \mathbf{S}_{n,i}^{k+1}(\boldsymbol{\theta}_{n,i}^\star -\boldsymbol{\theta}_{n,i}^{k+1}).
\end{align}
Using $\boldsymbol{r}_{n,i}^{k+1}=\boldsymbol{\theta}_{n,i}^{k+1}-\boldsymbol{\Theta}_{i}^{k+1}$, and $\boldsymbol{r}_{n,i}^\star=\boldsymbol{\theta}_{n,i}^\star-\boldsymbol{\Theta}_{i}^\star=0$ gives
\begin{align}
\nonumber &\sum_{n=1}^N \left[f_n(\boldsymbol{\theta}_n^{k+1}) -f_n(\boldsymbol{\theta}_n^\star)\right]\\
& \leq - \sum_{n=1}^N \sum_{i=1}^d \bm{\mu}_{n,i}^{k+1} \boldsymbol{r}_{n,i}^{k+1} + \sum_{n=1}^N \sum_{i=1}^d \mathbf{S}_{n,i}^{k+1}(\boldsymbol{\theta}_{n,i}^\star-\boldsymbol{\theta}_{n,i}^{k+1}).
\label{eql1}
\end{align}
and hence we have proved the statement (i).

\emph{Proof of statement (ii):}\\
	We note that for a saddle point  $(\boldsymbol{\Theta}^\star,\boldsymbol{\theta}^\star,\{\lamb_{n}^\star\}_{n})$ of   $\boldsymbol{\mathcal{L}}_{0}(\boldsymbol{\Theta}^\star,\{\boldsymbol{\theta}_n\}_{n}, \{\lamb_{n}\}_{n})$, it holds that, for all $n$, we have
		\begin{align}\label{temp}
	\boldsymbol{\mathcal{L}}_{0}( \boldsymbol{\Theta}^\star,\boldsymbol{\theta}^\star, \{\lamb_{n}^\star\}_{n} ) \leq \boldsymbol{\mathcal{L}}_{0}( \boldsymbol{\Theta}^{k+1},\boldsymbol{\{\theta}_n^{k+1}\}_{n}, \{\lamb_{n}^\star\}_{n}).
	\end{align}
Substituting the expression for the Lagrangian from \eqref{augmentedLagAG} on the both sides of \eqref{temp}, we get
\begin{align}	
\nonumber &\sum_{n=1}^N f_n(\boldsymbol{\theta}^\star)+\sum_{n=1}^N \sum_{i=1}^d \bm{\mu}_{n,i}^\star (\boldsymbol{\theta}_{n,i}^\star-\boldsymbol{\Theta}_i^\star) \\ & \leq \sum_{n=1}^N f_n(\boldsymbol{\theta}_n^{k+1})+\sum_{n=1}^N \sum_{i=1}^d \bm{\mu}_{n,i}^\star (\boldsymbol{\theta}_{n,i}^{k+1}- \boldsymbol{\Theta}_i^{k+1}).
\end{align}
Using $\bbr_{n,i}^{k+1}=\boldsymbol{\theta}_{n,i}^{k+1}- \boldsymbol{\Theta}_i^{k+1}$, and $\bbr_{n,i}^\star=\boldsymbol{\theta}_{n,i}^\star- \boldsymbol{\Theta}_i^\star=0$ gives
\begin{align}
\sum_{n=1}^N \left[f_n(\boldsymbol{\theta}_n^{k+1})- f_n(\boldsymbol{\theta}^\star)\right] \geq -\sum_{n=1}^N \sum_{i=1}^d \bm{\mu}_{n,i}^\star \bbr_{n,i}^{k+1}.
\label{eq3a}
\end{align}
which proves the statement $(ii)$. 

Finally, combining the statements $(i)$ and $(ii)$ completes the proof. 

\subsection{Proof of Theorem \ref{theorem}}\label{sec:them1}
The proof relies on using the lower and upper bounds derived in Lemma \ref{lemma:first} to show the decrease in the optimality gap. To this end, we start by multiplying both Eqs. (\ref{eql1}) and (\ref{eq3a}) by 2, and then add them up to get
\small
\begin{align}
2 \sum_{n=1}^N \sum_{i=1}^d \left( \bm{\mu}_{n,i}^{k+1} -  \bm{\mu}_{n,i}^\star \right) \bbr_{n,i}^{k+1} + 2 \sum_{n=1}^N \sum_{i=1}^d \boldsymbol{S}_{n,i}^{k+1} \left(\boldsymbol{\theta}_{n,i}^{k+1} - \boldsymbol{\theta}_{m,i}^\star\right) \leq 0
\label{eqthm1}
\end{align}
\normalsize
Since $\bm{\mu}_{n,i}^{k+1} = \bm{\mu}_{n,i}^{k} + \rho |h_{n,i}^{k+1}|^2 \bbr_{n,i}^{k+1}$, then the first term can be re-written as
\small
\begin{align}
& \nonumber 2 \sum_{n=1}^N \sum_{i=1}^d \left( \bm{\mu}_{n,i}^{k+1} -  \bm{\mu}_{n,i}^\star \right) \bbr_{n,i}^{k+1} \\
&= 2~ \sum_{n=1}^N \sum_{i=1}^d \left( \bm{\mu}_{n,i}^{k} -  \bm{\mu}_{n,i}^\star \right) \bbr_{n,i}^{k+1} + 2 \rho \sum_{n=1}^N \sum_{i=1}^d |h_{n,i}^{k+1}|^2 \left(\bbr_{n,i}^{k+1}\right)^2.
\label{eqthm2}
\end{align}
\normalsize
Since $\bbr_{n,i}^{k+1} = \frac{1}{\rho |h_{n,i}^{k+1}|^2}\left(\bm{\mu}_{n,i}^{k+1} -\bm{\mu}_{n,i}^{k}\right)$, we can write
\begin{align}
& \nonumber 2 \sum_{n=1}^N \sum_{i=1}^d \left( \bm{\mu}_{n,i}^{k+1} -  \bm{\mu}_{n,i}^\star \right) \bbr_{n,i}^{k+1} \\
\nonumber & = \frac{2}{\rho} \sum_{n=1}^N \sum_{i=1}^d \frac{1}{|h_{n,i}^{k+1}|^2} \left( \bm{\mu}_{n,i}^{k} -  \bm{\mu}_{n,i}^\star \right) \left( \bm{\mu}_{n,i}^{k+1} -  \bm{\mu}_{n,i}^{k}\right) \\
&+ 2 \rho \sum_{n=1}^N \sum_{i=1}^d |h_{n,i}^{k+1}|^2 \left(\bbr_{n,i}^{k+1}\right)^2.
\end{align}
Using the fact that $\bm{\mu}_{n,i}^{k+1} - \bm{\mu}_{n,i}^{k} = \bm{\mu}_{n,i}^{k+1} - \bm{\mu}_{n,i}^\star + \bm{\mu}_{n,i}^\star - \bm{\mu}_{n,i}^{k}$, we get
\begin{align}
\nonumber &\frac{1}{|h_{n,i}^{k+1}|^2} \left( \bm{\mu}_{n,i}^{k} -  \bm{\mu}_{n,i}^\star \right) \left( \bm{\mu}_{n,i}^{k+1} -  \bm{\mu}_{n,i}^{k}\right)\\
& = \frac{1}{|h_{n,i}^{k+1}|^2} \left( \bm{\mu}_{n,i}^{k} -  \bm{\mu}_{n,i}^\star \right) \left( \bm{\mu}_{n,i}^{k+1} -  \bm{\mu}_{n,i}^\star\right) - \frac{(\bm{\mu}_{n,i}^{k} - \bm{\mu}_{n,i}^\star)^2}{|h_{n,i}^{k+1}|^2}.
\end{align}
Now, let's re-write $\rho |h_{n,i}^{k+1}|^2 \left(\bbr_{n,i}^{k+1}\right)^2$, using $\bbr_{n,i}^{k+1} = \frac{1}{\rho |h_{n,i}^{k+1}|^2}\left(\bm{\mu}_{n,i}^{k+1} -\bm{\mu}_{n,i}^{k}\right)$, as
\begin{align}
& \nonumber \rho |h_{n,i}^{k+1}|^2 \left(\bbr_{n,i}^{k+1}\right)^2 \\
\nonumber & = \frac{1}{\rho |h_{n,i}^{k+1}|^2}  (\bm{\mu}_{n,i}^{k+1} - \bm{\mu}_{n,i}^{k})^2 \\
\nonumber & = \frac{1}{\rho |h_{n,i}^{k+1}|^2} (\bm{\mu}_{n,i}^{k+1} - \bm{\mu}_{n,i}^\star)^2  + \frac{1}{\rho |h_{n,i}^{k+1}|^2} (\bm{\mu}_{n,i}^{k} - \bm{\mu}_{n,i}^\star)^2\\
&  - \frac{2}{\rho |h_{n,i}^{k+1}|^2} \left(\bm{\mu}_{n,i}^{k} - \bm{\mu}_{n,i}^\star\right) \left(\bm{\mu}_{n,i}^{k+1} - \bm{\mu}_{n,i}^\star\right),  
\end{align}
where we have used that $\bm{\mu}_{n,i}^{k+1} - \bm{\mu}_{n,i}^{k} = \bm{\mu}_{n,i}^{k+1} - \bm{\mu}_{n,i}^\star + \bm{\mu}_{n,i}^\star - \bm{\mu}_{n,i}^{k}$. Going back to Eq. (\ref{eqthm2}), we can write
\small
\begin{align}
& \nonumber 2 \sum_{n=1}^N \sum_{i=1}^d \left( \bm{\mu}_{n,i}^{k+1} -  \bm{\mu}_{n,i}^\star \right) \bbr_{n,i}^{k+1} \\
\nonumber & = \frac{1}{\rho}  \sum_{n=1}^N \sum_{i=1}^d \frac{1}{|h_{n,i}^{k+1}|^2}(\bm{\mu}_{n,i}^{k+1} - \bm{\mu}_{n,i}^\star)^2\\
& - \frac{1}{\rho}  \sum_{n=1}^N \sum_{i=1}^d \frac{1}{|h_{n,i}^{k+1}|^2} (\bm{\mu}_{n,i}^{k} - \bm{\mu}_{n,i}^\star)^2 + \rho  \sum_{n=1}^N \sum_{i=1}^d |h_{n,i}^{k+1}|^2 \left(\bbr_{n,i}^{k+1}\right)^2.
\label{eqthm5}
\end{align}
\normalsize
Now, let's examine the second term of Eq. (\ref{eqthm1})
\begin{align}
\nonumber & 2 \sum_{n=1}^N \sum_{i=1}^d \boldsymbol{S}_{n,i}^{k+1} (\boldsymbol{\theta}_{n,i}^{k+1}-\boldsymbol{\theta}_{n,i}^\star) \\
\nonumber & =  2 \rho \sum_{n=1}^N \sum_{i=1}^d |h_{n,i}^{k+1}|^2 (\boldsymbol{\Theta}_i^{k+1}-\boldsymbol{\Theta}_i^{k}) (\boldsymbol{\theta}_{n,i}^{k+1}-\boldsymbol{\theta}_{n,i}^\star) \\
\nonumber &= 2 \rho \sum_{n=1}^N \sum_{i=1}^d |h_{n,i}^{k+1}|^2 (\boldsymbol{\Theta}_i^{k+1}-\boldsymbol{\Theta}_i^{k}) \bbr_{n,i}^{k+1} \\
&+  2 \rho \sum_{n=1}^N \sum_{i=1}^d  |h_{n,i}^{k+1}|^2 (\boldsymbol{\Theta}_{i}^{k+1}-\boldsymbol{\Theta}_{i}^{k})(\boldsymbol{\Theta}_{i}^{k+1}-\boldsymbol{\theta}_{n,i}^\star).
\end{align}
Using $\boldsymbol{\Theta}_{i}^{k+1} - \boldsymbol{\theta}_{n,i}^\star = \boldsymbol{\Theta}_{i}^{k+1} -  \boldsymbol{\Theta}_{i}^{k} +  \boldsymbol{\Theta}_{i}^{k} - \boldsymbol{\theta}_{n,i}^\star$, we can write
\begin{align}
\nonumber & 2 \sum_{n=1}^N \sum_{i=1}^d \boldsymbol{S}_{n,i}^{k+1} (\boldsymbol{\theta}_{n,i}^{k+1}-\boldsymbol{\theta}_{n,i}^\star) \\
&\nonumber = 2 \rho \sum_{n=1}^N \sum_{i=1}^d |h_{n,i}^{k+1}|^2 (\boldsymbol{\Theta}_i^{k+1}-\boldsymbol{\Theta}_i^{k}) \bbr_{n,i}^{k+1}\\
&\nonumber + 2 \rho \sum_{n=1}^N \sum_{i=1}^d |h_{n,i}^{k+1}|^2  (\boldsymbol{\Theta}_i^{k+1}-\boldsymbol{\Theta}_i^{k})^2\\
&  + 2 \rho \sum_{n=1}^N \sum_{i=1}^d |h_{n,i}^{k+1}|^2  (\boldsymbol{\Theta}_i^{k+1}-\boldsymbol{\Theta}_i^{k}) (\boldsymbol{\Theta}_i^{k}-\boldsymbol{\theta}_{n,i}^\star).
\end{align}
Since $\boldsymbol{\Theta}_i^{k+1} - \boldsymbol{\Theta}_i^{k} = \boldsymbol{\Theta}_i^{k+1} - \boldsymbol{\theta}_{n,i}^\star + \boldsymbol{\theta}_{n,i}^\star - \boldsymbol{\Theta}_i^{k}$, then we get
\begin{align}
\nonumber& 2 \sum_{n=1}^N \sum_{i=1}^d \boldsymbol{S}_{n,i}^{k+1} (\boldsymbol{\theta}_{n,i}^{k+1}-\boldsymbol{\theta}_{n,i}^\star) \\
&\nonumber = 2 \rho \sum_{n=1}^N \sum_{i=1}^d |h_{n,i}^{k+1}|^2 (\boldsymbol{\Theta}_i^{k+1}-\boldsymbol{\Theta}_i^{k}) \bbr_{n,i}^{k+1} \\
&\nonumber + 2 \rho \sum_{n=1}^N \sum_{i=1}^d |h_{n,i}^{k+1}|^2 (\boldsymbol{\Theta}_i^{k+1}-\boldsymbol{\Theta}_i^{k})^2\\
\nonumber & - 2 \rho \sum_{n=1}^N \sum_{i=1}^d |h_{n,i}^{k+1}|^2 (\boldsymbol{\Theta}_i^{k}-\boldsymbol{\theta}_{n,i}^\star)^2 \\
& + 2 \rho \sum_{n=1}^N \sum_{i=1}^d |h_{n,i}^{k+1}|^2 (\boldsymbol{\Theta}_i^{k+1} - \boldsymbol{\theta}_{n,i}^\star)(\boldsymbol{\Theta}_i^{k} - \boldsymbol{\theta}_{n,i}^\star),
\label{eqthm3}
\end{align}
Now, let's focus the second term of Eq. (\ref{eqthm3}). Using $\boldsymbol{\Theta}_i^{k+1} - \boldsymbol{\Theta}_i^{k} = \boldsymbol{\Theta}_i^{k+1} - \boldsymbol{\theta}_{n,i}^\star + \boldsymbol{\theta}_{n,i}^\star - \boldsymbol{\Theta}_i^{k}$, we can write
\begin{align}
\nonumber & \rho \sum_{n=1}^N \sum_{i=1}^d |h_{n,i}^{k+1}|^2 (\boldsymbol{\Theta}_i^{k+1}-\boldsymbol{\Theta}_i^{k})^2\\
\nonumber & = \rho \sum_{n=1}^N \sum_{i=1}^d |h_{n,i}^{k+1}|^2 (\boldsymbol{\Theta}_i^{k+1}-\boldsymbol{\theta}_{n,i}^\star)^2 \\ \nonumber &+ \rho \sum_{n=1}^N \sum_{i=1}^d |h_{n,i}^{k+1}|^2 (\boldsymbol{\Theta}_i^{k}-\boldsymbol{\theta}_{n,i}^\star)^2 \\
&- 2 \rho \sum_{n=1}^N \sum_{i=1}^d |h_{n,i}^{k+1}|^2  (\boldsymbol{\Theta}_i^{k+1}-\boldsymbol{\theta}_{n,i}^\star) (\boldsymbol{\Theta}_i^{k}-\boldsymbol{\theta}_{n,i}^\star).
\end{align}
Replacing the last equation into Eq. (\ref{eqthm3}), we can write
\begin{align}
\nonumber& 2 \sum_{n=1}^N \sum_{i=1}^d \boldsymbol{S}_{n,i}^{k+1} (\boldsymbol{\theta}_{n,i}^{k+1}-\boldsymbol{\theta}_{n,i}^\star)  \\
\nonumber &= 2 \rho \sum_{n=1}^N \sum_{i=1}^d |h_{n,i}^{k+1}|^2 (\boldsymbol{\Theta}_i^{k+1}-\boldsymbol{\Theta}_i^{k}) \bbr_{n,i}^{k+1} \\
\nonumber & +  \rho \sum_{n=1}^N \sum_{i=1}^d |h_{n,i}^{k+1}|^2 (\boldsymbol{\Theta}_i^{k+1}-\boldsymbol{\Theta}_i^{k})^2  \\
\nonumber & +  \rho \sum_{n=1}^N \sum_{i=1}^d |h_{n,i}^{k+1}|^2 (\boldsymbol{\Theta}_i^{k+1}-\boldsymbol{\theta}_{n,i}^\star)^2  \\& -  \rho \sum_{n=1}^N \sum_{i=1}^d |h_{n,i}^{k+1}|^2  (\boldsymbol{\Theta}_i^{k}-\boldsymbol{\theta}_{n,i}^\star)^2.
\label{eqthm4}
\end{align}
Using Eqs. (\ref{eqthm5}) and (\ref{eqthm4}) in (\ref{eqthm1}), we get
\small
\begin{align}\label{eqdiff}
& \nonumber \frac{1}{\rho} \sum_{n=1}^N \sum_{i=1}^d \frac{(\bm{\mu}_{n,i}^{k+1} - \bm{\mu}_{n,i}^\star)^2}{|h_{n,i}^{k+1}|^2} - \frac{1}{\rho} \sum_{n=1}^N \sum_{i=1}^d \frac{(\bm{\mu}_{n,i}^{k} - \bm{\mu}_{n,i}^\star)^2}{|h_{n,i}^{k+1}|^2} \\ & \nonumber + \rho \sum_{n=1}^N \sum_{i=1}^d |h_{n,i}^{k+1}|^2 (\boldsymbol{\Theta}_i^{k+1}-\boldsymbol{\Theta}_i^{k})^2 + \rho \sum_{n=1}^N \sum_{i=1}^d |h_{n,i}^{k+1}|^2 \left(\bbr_{n,i}^{k+1}\right)^2\\
\nonumber & + \rho \sum_{n=1}^N \sum_{i=1}^d |h_{n,i}^{k+1}|^2 (\boldsymbol{\Theta}_i^{k+1}-\boldsymbol{\Theta}_{i}^\star)^2 - \rho \sum_{n=1}^N \sum_{i=1}^d |h_{n,i}^{k+1}|^2 (\boldsymbol{\Theta}_i^{k}-\boldsymbol{\Theta}_{i}^\star)^2  \\
& + 2 \rho \sum_{n=1}^N \sum_{i=1}^d |h_{n,i}^{k+1}|^2 (\boldsymbol{\Theta}_i^{k+1}-\boldsymbol{\Theta}_i^{k}) \bbr_{n,i}^{k+1} \leq 0
\end{align}
\normalsize
Defining the sequence that measures the difference in the optimality gap between iterations $k+1$ and $k$ as
\small
\begin{align}
\nonumber & W^{k+1} \\
\nonumber &= \frac{1}{\rho} \sum_{n=1}^N \sum_{i=1}^d \frac{(\bm{\mu}_{n,i}^{k+1} - \bm{\mu}_{n,i}^\star)^2}{|h_{n,i}^{k+1}|^2}  - \frac{1}{\rho} \sum_{n=1}^N \sum_{i=1}^d \frac{(\bm{\mu}_{n,i}^{k} - \bm{\mu}_{n,i}^\star)^2}{|h_{n,i}^{k+1}|^2} \\
& + \rho \sum_{n=1}^N \sum_{i=1}^d |h_{n,i}^{k+1}|^2 (\boldsymbol{\Theta}_i^{k+1}-\boldsymbol{\Theta}_{i}^\star)^2 - \rho \sum_{n=1}^N \sum_{i=1}^d |h_{n,i}^{k+1}|^2 (\boldsymbol{\Theta}_i^{k}-\boldsymbol{\Theta}_{i}^\star)^2  ,
\end{align}
\normalsize
then, Eq. \eqref{eqdiff} can be re-written as
\begin{align}\label{decrease}
W^{k+1} \leq - \rho \sum_{n=1}^N \sum_{i=1}^d |h_{n,i}^{k+1}|^2 \left(\bbr_{n,i}^{k+1}+ \boldsymbol{\Theta}_i^{k+1}-\boldsymbol{\Theta}_i^{k} \right)^2.
\end{align}
From \eqref{decrease}, we note that $W^{k+1}$ is negative since $\rho$ is a positive constant ($\rho > 0$) and the sum square term is positive. Hence, the optimality gap at iteration $k+1$ is non-increasing, which completes the proof. 
\subsection{Proof of Corollary \ref{corollary}}\label{sec:corr1}
Using the derivations made in Theorem \ref{theorem}, we can further show that, in the case of static channel, i.e. $h_{n,i}^{k+1} = h_{n,i}$, both the  primal and dual residuals converge to zero, i.e. $\underset{k \rightarrow \infty}{\lim} \bbr_{n,i}^{k+1} = 0$ and $\underset{k \rightarrow \infty}{\lim} \bbS_{n,i}^{k+1} = 0$. To this end, we start defining the Lyapunov function
\small
\begin{align}
V^k =\frac{1}{\rho} \sum_{n=1}^N \sum_{i=1}^d \frac{(\bm{\mu}_{n,i}^{k}-\bm{\mu}_{n,i}^\star)^2 }{|h_{n,i}|^2}+ \rho \sum_{n=1}^N \sum_{i=1}^d  |h_{n,i}|^2 (\boldsymbol{\Theta}_i^{k} -  \boldsymbol{\Theta}_i^\star)^2.
\label{lyapEq}
\end{align}
\normalsize
We can re-write \eqref{decrease} as
\small
\begin{align}\label{decrease2}
\nonumber &V^{k+1} - V^k \\
\nonumber & \leq - \rho \sum_{n=1}^N \sum_{i=1}^d |h_{n,i}|^2 (\bbr_{n,i}^{k+1})^2 - \rho \sum_{n=1}^N \sum_{i=1}^d |h_{n,i}|^2 \left(\boldsymbol{\Theta}_{i}^{k+1}-\boldsymbol{\Theta}_i^{k} \right)^2\\
& - 2 \rho \sum_{n=1}^N \sum_{i=1}^d |h_{n,i}|^2 \bbr_{n,i}^{k+1} \left(\boldsymbol{\Theta}_{i}^{k+1}-\boldsymbol{\Theta}_i^{k} \right)
\end{align}
\normalsize
Since $\boldsymbol{\Theta}_i^{k+1}$ minimizes $-\sum_{n=1}^N \bm{\mu}_{n,i}^{k+1} \boldsymbol{\Theta}_{i}$, and $\boldsymbol{\Theta}_i^{k}$ minimizes $-\sum_{n=1}^N \bm{\mu}_{n,i}^{k }\boldsymbol{\Theta}_{i}$, then, after summing over $i$ in both sides of each equation, we can write
\begin{align}
-\sum_{n=1}^N \sum_{i=1}^d  \bm{\mu}_{n,i}^{k+1} \boldsymbol{\Theta}_{i}^{k+1}
\leq  -\sum_{n=1}^N \sum_{i=1}^d  \bm{\mu}_{n,i}^{k+1} \boldsymbol{\Theta}_{i}^{k},
\label{eqthm6}
\end{align}
\begin{align}
-\sum_{n=1}^N \sum_{i=1}^d  \bm{\mu}_{n,i}^{k} \boldsymbol{\Theta}_{i}^{k}
\leq  -\sum_{n=1}^N \sum_{i=1}^d  \bm{\mu}_{n,i}^{k} \boldsymbol{\Theta}_{i}^{k+1}.
\label{eqthm7}
\end{align}
Adding Eqs. (\ref{eqthm6}) and (\ref{eqthm7}), we get
\begin{align}
\sum_{n=1}^N \sum_{i=1}^d ( \bm{\mu}_{n,i}^{k+1}- \bm{\mu}_{n,i}^{k}) (\boldsymbol{\Theta}_{i}^{k+1}-\boldsymbol{\Theta}_{i}^{k}) \geq 0.
\end{align}
Since $ \bm{\mu}_{n,i}^{k+1} =  \bm{\mu}_{n,i}^{k} + \rho |h_{n,i}|^2 \bbr_{n,i}^{k+1}$, then we get
\begin{align}
2 \rho \sum_{n=1}^N \sum_{i=1}^d |h_{n,i}|^2 (\boldsymbol{\Theta}_i^{k+1}-\boldsymbol{\Theta}_i^{k}) \bbr_{n,i}^{k+1} \geq 0.
\end{align}
Thus, using Eq. \eqref{decrease2}, and summing over the iterations from $k=1,\dots,K$, we get
\small
\begin{align}
\nonumber &\sum_{k=0}^K \left[\rho \sum_{n=1}^N \sum_{i=1}^d |h_{n,i}|^2 \left(\bbr_{n,i}^{k+1}\right)^2 + \rho \sum_{n=1}^N \sum_{i=1}^d |h_{n,i}|^2 (\boldsymbol{\Theta}_i^{k+1}-\boldsymbol{\Theta}_i^{k})^2 \right. \\
&\left. + 2\rho \sum_{n=1}^N \sum_{i=1}^d |h_{n,i}|^2 (\boldsymbol{\Theta}_i^{k+1}-\boldsymbol{\Theta}_i^{k}) \bbr_{n,i}^{k+1} \right] \leq V^0.
\end{align}
\normalsize
Taking the limit as $K \rightarrow \infty$, and using the fact that the terms of the serie on the left hand-side are positive, we obtain that the primal and dual residuals goes to zero as $k \rightarrow \infty$, i.e. $\underset{k \rightarrow \infty}{\lim} \bbr_{n,i}^{k+1} = \boldsymbol{0}$ and $\underset{k \rightarrow \infty}{\lim} \bbS_{n,i}^{k+1} = \boldsymbol{0}$.
Using the upper and lower bounds, \eqref{lem1Lower} and \eqref{lem1Upper}, derived in Lemma \ref{lemma:first} and the fact that both the primal and dual residuals goes to zero as $k \rightarrow \infty$, we get that the optimal gap also goes to zero as $k \rightarrow \infty$, i.e., $\underset{k \rightarrow \infty}{\lim} \sum_{n=1}^N f_n(\boldsymbol{\theta}_n^{k}) =  \sum_{n=1}^N f_n(\boldsymbol{\theta}_n^\star)$, finalizing the proof.

\bibliographystyle{IEEEtran}

\bibliography{refsAADMM.bib}

\begin{thebibliography}{10}
\providecommand{\url}[1]{#1}
\csname url@samestyle\endcsname
\providecommand{\newblock}{\relax}
\providecommand{\bibinfo}[2]{#2}
\providecommand{\BIBentrySTDinterwordspacing}{\spaceskip=0pt\relax}
\providecommand{\BIBentryALTinterwordstretchfactor}{4}
\providecommand{\BIBentryALTinterwordspacing}{\spaceskip=\fontdimen2\font plus
\BIBentryALTinterwordstretchfactor\fontdimen3\font minus
  \fontdimen4\font\relax}
\providecommand{\BIBforeignlanguage}[2]{{%
\expandafter\ifx\csname l@#1\endcsname\relax
\typeout{** WARNING: IEEEtran.bst: No hyphenation pattern has been}%
\typeout{** loaded for the language `#1'. Using the pattern for}%
\typeout{** the default language instead.}%
\else
\language=\csname l@#1\endcsname
\fi
#2}}
\providecommand{\BIBdecl}{\relax}
\BIBdecl

\bibitem{Brendan17}
H.~B. McMahan, E.~Moore, D.~Ramage \emph{et~al.}, ``Communication-efficient
  learning of deep networks from decentralized data,'' \emph{In Proceedings of
  Artificial Intelligence and Statistics, Fort Lauderdale, FL, USA}, April
  2017.

\bibitem{pap:jakub16}
\BIBentryALTinterwordspacing
J.~Konecny, H.~B. McMahan, F.~X. Yu, P.~Richtarik, A.~T. Suresh, and D.~Bacon,
  ``Federated learning: strategies for improving communication efficiency,'' in
  \emph{Proc. of NIPS Wksp. PMPML}, Barcelona, Spain, December 2016. [Online].
  Available: \url{https://arxiv.org/abs/1610.05492}
\BIBentrySTDinterwordspacing

\bibitem{Google:FL19}
P.~Kairouz, H.~B. McMahan, B.~Avent, A.~Bellet, M.~Bennis \emph{et~al.},
  ``Advances and open problems in federated learning,'' \emph{arXiv preprint
  arXiv:1912.04977}, 2019.

\bibitem{park2018wireless}
J.~Park, S.~Samarakoon, M.~Bennis, and M.~Debbah, ``Wireless network
  intelligence at the edge,'' \emph{Proceedings of the IEEE}, vol. 107, no.~11,
  pp. 2204--2239, October 2019.

\bibitem{FL_Nishio}
\BIBentryALTinterwordspacing
T.~Nishio and R.~Yonetani, ``Client selection for federated learning with
  heterogeneous resources in mobile edge,'' \emph{In Proc. Int'l Conf. Commun.
  (ICC), Shanghai, China}, May 2019. [Online]. Available:
  \url{http://arxiv.org/abs/1804.08333}
\BIBentrySTDinterwordspacing

\bibitem{Wang:2019aa}
S.~{Wang}, T.~{Tuor}, T.~{Salonidis}, K.~K. {Leung}, C.~{Makaya}, T.~{He}, and
  K.~{Chan}, ``Adaptive federated learning in resource constrained edge
  computing systems,'' \emph{IEEE Journal on Selected Areas in Communications},
  vol.~37, no.~6, pp. 1205--1221, June 2019.

\bibitem{YangQuekPoor:2019aa}
H.~H. Yang, Z.~Liu, T.~Q.~S. Quek, and H.~V. Poor, ``Scheduling policies for
  federated learning in wireless networks,'' \emph{arXiv preprint arXiv:
  1908.06287}, 2019.

\bibitem{Chen:20019aa}
M.~Chen, Z.~Yang, W.~Saad, C.~Yin, H.~V. Poor, and S.~Cui, ``A joint learning
  and communications framework for federated learning over wireless networks,''
  \emph{arXiv preprint arXiv: 1909.07972}, 2019.

\bibitem{Amiri:SPAWC19}
M.~M. Amiri and D.~Gunduz, ``Over-the-air machine learning at the wireless
  edge,'' \emph{Proc. IEEE International Workshop on Signal Processing Advances
  in Wireless Communications (SPWAC), Cannes, France}, July 2019.

\bibitem{Zhu:19}
G.~Zhu, Y.~Wang, and K.~Huang, ``Broadband analog aggregation for low-latency
  federated edge learning,'' \emph{arXiv preprint arXiv: 1812.11494}.

\bibitem{Sery:19}
T.~Sery and K.~Cohen, ``On analog gradient descent learning over multiple
  access fading channels,'' \emph{arXiv preprint arXiv: 1908.07463}.

\bibitem{Zhu:20}
G.~Zhu, Y.~Du, D.~Dunduz, and K.~Huang, ``One-bit over-the-air aggregation for
  communication-efficient federated edge learning: Design and convergence
  analysis,'' \emph{arXiv preprint arXiv: 2001.05713}.

\bibitem{park2020:cml}
J.~Park, S.~Samarakoon, A.~Elgabli, J.~Kim, M.~Bennis, S.-L. Kim, and
  M.~Debbah, ``Communication-efficient and distributed learning over wireless
  networks: Principles and applications.''\hskip 1em plus 0.5em minus
  0.4em\relax arXiv preprint arXiv:2008.02608, 2020.

\bibitem{Matt:CCS15}
\BIBentryALTinterwordspacing
M.~Fredrikson, S.~Jha, and T.~Ristenpart, ``Model inversion attacks that
  exploit confidence information and basic countermeasures,'' in
  \emph{Proceedings of the 22nd ACM SIGSAC Conference on Computer and
  Communications Security}, ser. CCS '15.\hskip 1em plus 0.5em minus
  0.4em\relax New York, NY, USA: Association for Computing Machinery, 2015, pp.
  1322--1333. [Online]. Available:
  \url{https://doi.org/10.1145/2810103.2813677}
\BIBentrySTDinterwordspacing

\bibitem{Hitaj:CCS17}
\BIBentryALTinterwordspacing
B.~Hitaj, G.~Ateniese, and F.~Perez-Cruz, ``Deep models under the gan:
  Information leakage from collaborative deep learning,'' in \emph{Proceedings
  of the 2017 ACM SIGSAC Conference on Computer and Communications Security},
  ser. CCS '17.\hskip 1em plus 0.5em minus 0.4em\relax New York, NY, USA:
  Association for Computing Machinery, 2017, pp. 603--618. [Online]. Available:
  \url{https://doi.org/10.1145/3133956.3134012}
\BIBentrySTDinterwordspacing

\bibitem{boyd2011distributed}
S.~Boyd, N.~Parikh, E.~Chu, B.~Peleato, J.~Eckstein \emph{et~al.},
  ``Distributed optimization and statistical learning via the alternating
  direction method of multipliers,'' \emph{Foundations and
  Trends{\textregistered} in Machine learning}, vol.~3, no.~1, pp. 1--122,
  2011.

\bibitem{deng2017parallel}
W.~Deng, M.-J. Lai, Z.~Peng, and W.~Yin, ``Parallel multi-block admm with
  $o(1/k)$ convergence,'' \emph{Journal of Scientific Computing}, vol.~71,
  no.~2, pp. 712--736, 2017.

\bibitem{glowinski1975approximation}
R.~Glowinski and A.~Marroco, ``Sur l'approximation, par {\'e}l{\'e}ments finis
  d'ordre un, et la r{\'e}solution, par p{\'e}nalisation-dualit{\'e} d'une
  classe de probl{\`e}mes de dirichlet non lin{\'e}aires,'' \emph{ESAIM:
  Mathematical Modelling and Numerical Analysis-Mod{\'e}lisation
  Math{\'e}matique et Analyse Num{\'e}rique}, vol.~9, no.~R2, pp. 41--76, 1975.

\bibitem{Ozdemir:ST07}
M.~K. {Ozdemir} and H.~{Arslan}, ``Channel estimation for wireless ofdm
  systems,'' \emph{IEEE Communications Surveys Tutorials}, vol.~9, no.~2, pp.
  18--48, 2007.

\bibitem{Giannakis:04}
{Qingwen Liu}, {Shengli Zhou}, and G.~B. {Giannakis}, ``Cross-layer combining
  of adaptive modulation and coding with truncated arq over wireless links,''
  \emph{IEEE Transactions on Wireless Communications}, vol.~3, no.~5, pp.
  1746--1755, 2004.

\bibitem{Heath:Mag02}
S.~{Catreux}, V.~{Erceg}, D.~{Gesbert}, and R.~W. {Heath}, ``Adaptive
  modulation and mimo coding for broadband wireless data networks,'' \emph{IEEE
  Communications Magazine}, vol.~40, no.~6, pp. 108--115, 2002.

\bibitem{3GPP:Rel15}
3GPP, ``Ts 38.211 v15.2.0 release 15tr 38.802 v14.1.0,'' \emph{tech. rep.},
  June 2017.

\bibitem{zhang2018admm}
C.~Zhang, M.~Ahmad, and Y.~Wang, ``Admm based privacy-preserving decentralized
  optimization,'' \emph{IEEE Transactions on Information Forensics and
  Security}, vol.~14, no.~3, pp. 565--580, 2018.

\bibitem{Torgo:14}
\BIBentryALTinterwordspacing
L.~Torgo, ``Regression datasets,'' 2014. [Online]. Available:
  \url{https://www.dcc.fc.up.pt/~ltorgo/Regression/DataSets.html}
\BIBentrySTDinterwordspacing

\bibitem{LeCun:MNiST}
Y.~LeCun and C.~Cortes, ``{MNIST} handwritten digit database,'' 2010.

\bibitem{Elbamby:19}
M.~S. {Elbamby}, C.~{Perfecto}, C.~{Liu}, J.~{Park}, S.~{Samarakoon},
  X.~{Chen}, and M.~{Bennis}, ``Wireless edge computing with latency and
  reliability guarantees,'' \emph{Proceedings of the IEEE}, vol. 107, no.~8,
  pp. 1717--1737, Aug 2019.

\bibitem{Amiri:2019aa}
M.~M. Amiri and D.~G\"{u}nd\"{u}z, ``Federated learning over wireless fading
  channels,'' \emph{IEEE Transactions on Wireless Communications}, vol.~19,
  no.~5, pp. 3546--3557, 2020.

\bibitem{Donoho18914}
\BIBentryALTinterwordspacing
D.~L. Donoho, A.~Maleki, and A.~Montanari, ``Message-passing algorithms for
  compressed sensing,'' \emph{Proceedings of the National Academy of Sciences},
  vol. 106, no.~45, pp. 18\,914--18\,919, 2009. [Online]. Available:
  \url{https://www.pnas.org/content/106/45/18914}
\BIBentrySTDinterwordspacing

\bibitem{elgabli2019gadmm}
A.~Elgabli, J.~Park, A.~S. Bedi, M.~Bennis, and V.~Aggarwal, ``{GADMM}: Fast
  and communication efficient framework for distributed machine learning,''
  \emph{Journal of Machine Learning Research (JMLR)}, vol.~21, no.~76, pp.
  1--39, 2020.

\bibitem{benissaid2020}
C.~{Ben Issaid}, A.~Elgabli, J.~Park, and M.~Bennis, ``Communication efficient
  distributed learning with censored, quantized, and generalized group
  {ADMM},'' \emph{arXiv preprint arXiv:2009.06459}, 2020.

\end{thebibliography}

\end{document}